\documentclass[]{jingdong}

\usepackage{amsfonts}
\usepackage{nicefrac}
\usepackage{enumitem}
\usepackage{pifont}
\newcommand{\cmark}{\textcolor{green!70!black}{\ding{51}}}
\newcommand{\xmark}{\textcolor{red}{\ding{55}}}
\usepackage{amsmath,amssymb}
\usepackage{amsthm}
\usepackage{algorithm}
\usepackage{algpseudocode}
\usepackage{colortbl}
\usepackage[utf8]{inputenc}
\microtypesetup{expansion=false}
\graphicspath{{figures/}}
\title{Self-Distilled RLVR}

\author[1,2,*]{Chenxu Yang}
\author[1,2,*]{Chuanyu Qin}
\author[3,*]{Qingyi Si}
\author[1,2]{Minghui Chen}
\author[1,2]{Naibin Gu}
\author[1,2]{Dingyu Yao}
\author[1,2,\dagger]{Zheng Lin}
\author[1]{Weiping Wang}
\author[3]{Jiaqi Wang}
\author[3]{Nan Duan}

\affiliation[1]{Institute of Information Engineering, Chinese Academy of Sciences, Beijing, China}
\affiliation[2]{School of Cyber Security, University of Chinese Academy of Sciences, Beijing, China}
\affiliation[3]{JD.COM}

\contribution[*]{Equal contribution.}
\contribution[\dagger]{Corresponding author.}
\checkdata[Email]{\email{yangchenxu@iie.ac.cn}; \email{linzheng@iie.ac.cn}; \email{siqingyi.phoebus@jd.com}}

\abstract{On-policy distillation (OPD) has become a popular training paradigm in the LLM community. This paradigm selects a larger model as the teacher to provide dense, fine-grained signals for each sampled trajectory, in contrast to reinforcement learning with verifiable rewards (RLVR), which only obtains sparse signals from verifiable outcomes in the environment. Recently, the community has explored on-policy self-distillation (OPSD), where the same model serves as both teacher and student, with the teacher receiving additional privileged information such as reference answers to enable self-evolution. This paper demonstrates that learning signals solely derived from the privileged teacher result in severe information leakage and unstable long-term training. Accordingly, we identify the optimal niche for self-distillation and propose \textbf{RLSD} (\textbf{RL}VR with \textbf{S}elf-\textbf{D}istillation). Specifically, we leverage self-distillation to obtain token-level policy differences for determining fine-grained update magnitudes, while continuing to use RLVR to derive reliable update directions from environmental feedback (e.g., response correctness). This enables RLSD to simultaneously harness the strengths of both RLVR and OPSD, achieving a higher convergence ceiling and superior training stability.}

\begin{document}

\maketitle
\enlargethispage{1mm}
\vspace{-20mm}

\begin{figure*}[!b]
  \centerline{\includegraphics[scale=0.39]{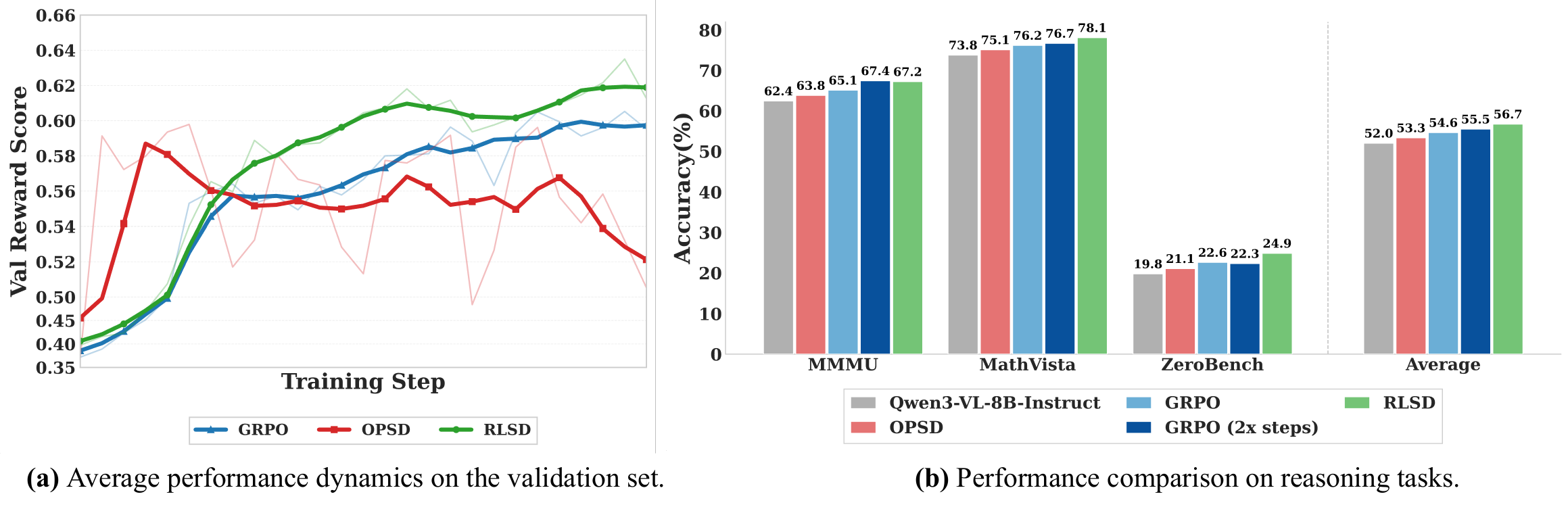}}
  \vspace{-0.1cm}
  \caption{Performance on the trained Qwen3-VL-8B-Instruct model. In (a), OPSD reaches its peak performance early and degrades, whereas RLSD inherits the training stability of GRPO while achieving a higher convergence ceiling. In (b), GRPO and RLSD report results at 200 training steps, while GRPO (2$\times$steps) reports results at 400 steps; RLSD at 200 steps already surpasses GRPO trained for twice as many steps, demonstrating faster convergence. }
  \label{fig-performance}
\end{figure*}

\newpage
\section{Introduction}
\label{sec:intro}
Reinforcement learning with verifiable rewards (RLVR) methods such as GRPO~\citep{grpo} have become a central paradigm for training large reasoning models~\citep{deepseekr1,kimiteam2026kimik25visualagentic}, where each trajectory receives only a single scalar signal determined by the outcome from the environment. On-policy distillation (OPD)~\citep{agarwal2024gkd,lu2025onpolicy} complements this by leveraging a stronger teacher model to provide dense, token-level logits as learning signals along the student's own sampled trajectories, enriching the trajectory-level supervision to the token level and thereby achieving faster convergence. Recent work has shown that OPD from advanced teachers can match or even outperform RLVR~\citep{mimo2025flash}, establishing it as an equally compelling paradigm (see Table~\ref{tab:unified} for a systematic comparison).

Despite its effectiveness, OPD relies on a separate, typically much larger teacher model, incurring substantial computational overhead. Moreover, since OPD requires computing token-level distributions over a shared vocabulary, the teacher and student models must share the same vocabulary, significantly reducing the practical scalability of this paradigm. On-policy self-distillation (OPSD)~\citep{zhao2026opsd,hbotter2026reinforcementlearningselfdistillation_sdpo} offers an appealing alternative: a single model serves as both teacher and student, where the teacher is the same model conditioned on \textit{privileged information} $r$ (such as a verified reasoning trace or environmental feedback) and the student operates solely on the input query. OPSD achieves several-fold improvements in token efficiency over GRPO without requiring any external model. However, we prove this efficiency gain is fragile: as illustrated by the red curve in Figure~\ref{fig-performance}(a), performance peaks early and subsequently deteriorates, accompanied by systematic privileged information leakage, where the model explicitly appeals to an invisible reference solution during inference despite never having access to it (see Figure~\ref{fig:case} for a representative example).
These phenomena raise a natural question: \textit{why does OPD work while OPSD fails?}



We trace the answer to a structural distinction between the two settings. In OPD, teacher and student observe the same input (information-symmetric), so the teacher's dense signal reflects superior reasoning under shared information access. In OPSD, the teacher conditions on privileged information that the student cannot observe (information-asymmetric), creating a fundamental mismatch. We prove that this asymmetry renders the OPSD objective \textit{ill-posed}: it contains an irreducible mutual information gap $I(Y_t; R \mid X, Y_{<t}) > 0$ that the student can never eliminate, regardless of its capacity (Theorem~\ref{thm:kl-decomp}). At the gradient level, we show that while the expected OPSD gradient is benign, the per-sample gradients carry an $r$-specific deviation whose variance is proportional to this mutual information. These theoretical findings are directly corroborated by our diagnostic experiments: the on-policy KL divergence between teacher and student stagnates under OPSD with no sustained reduction, whereas the same metric decreases steadily under OPD, as shown in Figure~\ref{fig:pilot}. Early in training, the beneficial gradient component dominates, producing rapid improvement; as the student approaches the teacher's marginal distribution, the deviation takes over, and its path-dependent accumulation drives the model toward encoding $x \to r$ correlations in its parameters. This two-phase dynamic precisely accounts for the observed pattern of early gains followed by progressive degradation. 


\begin{table}[!t]
\centering
\caption{Comparison of training paradigms for post-training LLMs.}
\vspace{0.1cm}
\label{tab:unified}
\small
\begin{tabular}{lccccc}
\toprule
\textbf{Method} & \textbf{Trajectory} & \textbf{Efficiency}  & \textbf{Leakage Risk} & \textbf{Signal} & \textbf{Direction Anchoring} \\
\midrule
SFT & \xmark Off-policy &  \cmark High  & \cmark N/A & \cmark \ Rich & \xmark Teacher \\
RLVR (GRPO) & \cmark On-policy &  \cmark High  & \cmark N/A & \xmark \ Weak & \cmark Environment  \\
OPD & \cmark On-policy & \xmark Low  &  \cmark N/A & \cmark \ Rich & \xmark Teacher \\
OPSD & \cmark On-policy & \cmark High & \xmark Severe & \cmark \ Rich & \xmark Teacher \\
\textbf{RLSD (Ours)} & \cmark On-policy & \cmark High & \cmark N/A & \cmark \ Rich & \cmark Environment \\
\bottomrule
\end{tabular}
\end{table}

Our analysis pinpoints the root cause: in all distribution-matching formulations, the teacher's privileged evaluation $P_T(y_t \mid r)$ enters the gradient \textit{direction}, making leakage structurally unavoidable regardless of how the distillation target is compressed. Yet the evidence ratio $P_T(y_t)/P_S(y_t)$ also carries a useful signal: it measures how much the privileged information revises the model's belief about each token. The challenge is therefore not to discard this signal, but to change how it is used. A key insight underlying our design is that the signals governing update \textit{direction} and update \textit{magnitude} have asymmetric requirements: \textbf{the directional signal can be sparse but must be reliable}, as an erroneous direction harms the policy; \textbf{the magnitude signal, by contrast, benefits from being as dense as possible to enable fine-grained discrimination among tokens}. 

We propose \textbf{RL}VR with \textbf{S}elf-\textbf{D}istillation (\textbf{RLSD}), which instantiates this principle by repurposing the teacher from a generative target to a \textit{magnitude evaluator}. Specifically, the environment reward determines the \textit{direction} of each token's update (reinforcement or penalization), while the teacher's evidence ratio modulates only the \textit{magnitude}. This decoupling anchors gradient directions exclusively to the reliable environment reward, while retaining the teacher's dense, token-level assessment for fine-grained credit discrimination across token positions. Notably, prior approaches to token-level credit assignment, such as value-function estimation in PPO~\citep{schulman2017proximalpolicyoptimizationalgorithmsppo} and various credit assignment methods~\citep{xie2025unlocking,li2026outcome}, typically require training auxiliary networks or incur substantial additional overhead, yet still produce noisy estimates. In contrast, self-distillation provides a natural and virtually cost-free source of per-token credit information, requiring only a single additional forward pass. As summarized in Table~\ref{tab:unified}, RLSD is the only paradigm that simultaneously achieves on-policy training, high token efficiency, rich update signals, and environment-anchored optimization, while serving as a drop-in replacement for the uniform advantage in standard GRPO without any auxiliary loss or model. Our contributions are as follows:

\begin{itemize}[nosep]
    \item We identify the root cause of OPSD's failure through controlled experiments and formal analysis, proving that distribution matching under information asymmetry induces an irreducible gap that drives privileged information leakage through the gradient structure.
    \item We propose RLSD, a new training paradigm that unifies the strengths of RLVR and OPSD: the reliable environment reward governs update directions, while the privileged teacher provides rich, token-level modulation of update magnitudes.
    \item Extensive experiments demonstrate that RLSD achieves the best average accuracy across five multimodal reasoning benchmarks, outperforming the base LLM by 4.69\%, sustaining improvement beyond the point where self-distillation degrades.
\end{itemize}

\section{Preliminaries}
\label{sec:prelim}
 
\noindent \textbf{GRPO.}
Consider a language model $\pi_\theta$ trained to solve reasoning tasks.
Given a question $x$, the model generates a response $y = (y_1, \dots, y_T)$ autoregressively.
In the RLVR setting, a verifier provides a binary reward $R(x, y) \in \{0, 1\}$ indicating whether the response is correct.
Group Relative Policy Optimization (GRPO)~\citep{grpo} samples a group of $G$ responses $\{y^{(1)}, \dots, y^{(G)}\}$ from the current policy for each question $x$, and computes a sequence-level advantage for each response relative to the group:
\begin{equation}
    A^{(i)} = \frac{R(x, y^{(i)}) - \mu_G}{\sigma_G},
\end{equation}
where $\mu_G$ and $\sigma_G$ are the mean and standard deviation of the rewards within the group.
The policy is then updated via a clipped surrogate objective:
\begin{equation}
    \mathcal{L}_{\text{GRPO}}(\theta) = \mathbb{E}\left[\frac{1}{G}\sum_{i=1}^{G}\frac{1}{|y^{(i)}|}\sum_{t=1}^{|y^{(i)}|} \min\!\left(\rho_t^{(i)} A^{(i)},\; \text{clip}\!\left(\rho_t^{(i)}, 1-\epsilon, 1+\epsilon\right) A^{(i)}\right)\right],
    \label{eq:grpo}
\end{equation}
where $\rho_t^{(i)} = \pi_\theta(y_t^{(i)} \mid x, y_{<t}^{(i)}) / \pi_{\theta_{\text{old}}}(y_t^{(i)} \mid x, y_{<t}^{(i)})$ is the importance sampling ratio between the current and old policies.
A key limitation of GRPO is that all tokens within a response share the same advantage $A^{(i)}$, as the reward signal is provided only at the sequence level.

\noindent \textbf{OPD and OPSD.}
On-Policy Distillation (OPD)~\citep{agarwal2024gkd} addresses the sparse reward problem by having the student $\pi_\theta$ sample its own trajectories while a separate, typically larger, teacher model $\pi_{\hat{\theta}}$ provides dense token-level supervision.
While effective, maintaining a distinct teacher model throughout training introduces significant computational overhead.
On-Policy Self-Distillation (OPSD)~\citep{zhao2026opsd,hbotter2026reinforcementlearningselfdistillation_sdpo} eliminates this requirement by instantiating both roles from a single model $\pi_\theta$: the teacher gains its informational advantage not from greater capacity, but from conditioning on privileged information $r$ (e.g., a verified reasoning trace).

Both methods share the same training paradigm.
Given a dataset $\mathcal{S} = \{(x_i, r_i)\}_{i=1}^{N}$, the student generates on-policy rollouts $\hat{y} \sim P_S(\cdot \mid x)$, and training minimizes the per-token divergence between teacher and student distributions along the student's trajectory.
The two frameworks differ only in how the teacher is defined:
\begin{align}
    \text{Student:} \quad & P_S(\cdot \mid y_{<t}) \triangleq \pi_\theta(\cdot \mid x,\, y_{<t}), \\
    \text{OPD Teacher:} \quad & P_T(\cdot \mid y_{<t}) \triangleq \pi_{\hat{\theta}}(\cdot \mid x,\, y_{<t}), \\
    \text{OPSD Teacher:} \quad & P_T(\cdot \mid y_{<t}) \triangleq \pi_\theta(\cdot \mid x,\, r,\, y_{<t}).
\end{align}
The shared training objective is:
\begin{equation}
    \mathcal{L}_{\text{OP(S)D}}(\theta) = \mathbb{E}_{(x,r)\sim\mathcal{S}}\;\mathbb{E}_{\hat{y}\sim P_S(\cdot|x)}\left[\frac{1}{|\hat{y}|}\sum_{t=1}^{|\hat{y}|} D\!\left(P_T \;\|\; P_S\right)\right],
    \label{eq:opsd}
\end{equation}
where $D$ is a divergence measure such as the generalized Jensen-Shannon divergence.
Gradients are backpropagated only through $P_S$, while $P_T$ serves as a fixed target.
OPSD achieves high token efficiency by providing dense, per-token supervision without requiring an external teacher model.

\begin{figure}[t]
    \centering
    \begin{tcolorbox}[
        colback=red!8!white,      
        colframe=red!75!black,    
        arc=4pt,                  
        boxrule=1pt,              
        left=8pt, right=8pt,      
        top=8pt, bottom=8pt,      
    ]
    
    
    
    \textbf{OPSD-tuned Model Output (\xmark):} \\
    I need to determine if the sample mean is within $\$1$ of the population mean. To do this, I must first calculate both means.  \textcolor{blue!80!black}{ $\dots$ \textbf{<Omitted>} $\dots$ }   Given that the \textcolor{red}{\textbf{reference solution}}  uses 9 values, I will assume the sample has 9 prices  \textcolor{blue!80!black}{$\dots$ \textbf{<Omitted>}$\dots$}   But wait, \textcolor{red}{\textbf{the reference solution says}}  "No", which contradicts my calculation. Let me re-examine. $\dots$ 
    

    \end{tcolorbox}
    
    \caption{A representative example illustrating the privileged information leakage exhibited by the OPSD-trained model, where the model appeals to an invisible reference solution during inference.}
    \label{fig:case}
\end{figure}

\begin{figure*}[t]
  \centerline{\includegraphics[scale=0.46]{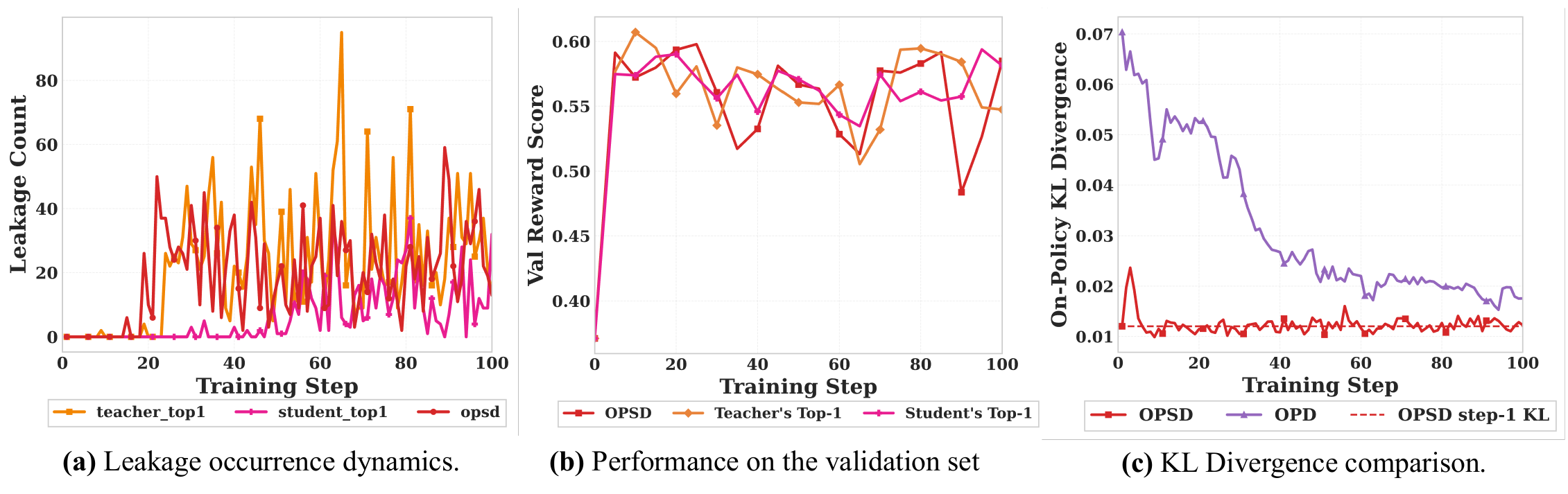}}
  \vspace{-0.1cm}
  \caption{Leakage, KL divergence, and validation performance of OPSD and its ablated variants. }
  \label{fig:pilot}
\end{figure*}

\newtheorem{example}{Example}
\newtheorem{theorem}{Theorem}

\newtheorem{lemma}{Lemma}
\newtheorem{corollary}{Corollary}
\newtheorem{definition}{Definition}
\newtheorem{proposition}{Proposition}
\newtheorem{assumption}{Assumption}
\newtheorem{remark}{Remark}


\section{Why Does OPD Work While OPSD Fails?}

\subsection{Empirical Observations: Leakage and Performance Degradation}
\label{sec:observations}

Before the formal analysis in \S\ref{sec:ill-posed}, we document the empirical phenomena that motivate it.

\noindent \textbf{Privileged information leakage.}
We first observe that OPSD-trained models systematically reference privileged information that is unavailable at inference time. Figure~\ref{fig:case} presents a representative example in which the model explicitly appeals to an invisible ``reference solution'' during generation. Such behavior is not an isolated failure. As we quantify in the following analysis, this leakage intensifies progressively over the course of training.

\noindent \textbf{Performance degradation.}
Figure~\ref{fig:pilot}(a) tracks the frequency of privileged information leakage over 100 training steps and reveals a monotonically increasing trend: the model becomes progressively more reliant on information it cannot access at test time. Figure~\ref{fig:pilot}(b) reports the corresponding validation accuracy, which peaks within the first 10\textendash 20 steps and subsequently declines, consistent with the escalating leakage.

\noindent \textbf{KL divergence stagnation.}
A further diagnostic contrasts the on-policy KL divergence between teacher and student for OPD and OPSD (Figure~\ref{fig:pilot}(c)). Under OPD, the KL divergence decreases steadily throughout training, reflecting genuine convergence. Under OPSD, the divergence drops briefly in the first few steps but then plateaus at a level comparable to its initial value, exhibiting no sustained reduction. This stagnation suggests the existence of an irreducible gap in the OPSD objective that prevents meaningful convergence, a hypothesis we formalize in the next subsection.

These observations raise two questions: \textit{why} does OPSD induce leakage, and \textit{why} does it fail to reduce the KL divergence? We address both in the following theoretical analysis.

\subsection{OPSD's Failure: An Ill-Posed Objective}
\label{sec:ill-posed}

We formalize the structural deficiency of the distribution matching paradigm underlying OPSD and related methods. Our analysis proceeds at two levels: the objective function (\S\ref{sec:ill-posed-obj}) and the gradient dynamics (\S\ref{sec:gradient-structure}).

\subsubsection{The Irreducible Mutual Information Gap}
\label{sec:ill-posed-obj}

Let $r$ denote the privileged information, drawn from the conditional distribution $P(r \mid x)$. Since any given question $x$ admits multiple semantically valid reasoning paths, $P(r \mid x)$ is a non-degenerate distribution with non-zero entropy. Even when each training instance $x_i$ is paired with a single reference trace $r_i$, from the epistemic perspective of the student, which can neither observe $r$ nor deterministically derive it from $x$, the privileged information remains an uncertain latent variable, and $P(r \mid x)$ should be treated as non-degenerate within the probabilistic modeling framework.

An optimal student policy that cannot condition on $r$ should recover the \textit{marginal} teacher distribution via the law of total probability:
\begin{equation}
    P_S^*(y_t \mid x, y_{<t}) = \mathbb{E}_{r \sim P(r \mid x, y_{<t})}\!\left[P_T(y_t \mid x, r, y_{<t})\right].
\end{equation}
Letting $\bar{P}_T(y_t) \triangleq \mathbb{E}_r[P_T(y_t \mid x, r, y_{<t})]$ denote this marginal, the \textit{ideal} distillation objective is:
\begin{equation}
    \mathcal{L}^*(\theta) = \mathbb{E}_{x}\!\left[D_{\mathrm{KL}}\!\left(\bar{P}_T(\cdot) \;\Big\|\; P_S(\cdot \mid x)\right)\right].
    \label{eq:ideal-obj}
\end{equation}

However, the OPSD objective enforces a \textit{per-sample} match $P_S(\cdot \mid x) \to P_T(\cdot \mid x, r)$ for each concrete pair $(x, r)$:
\begin{equation}
    \mathcal{L}_{\text{OPSD}}(\theta) = \mathbb{E}_{x}\;\mathbb{E}_{r \sim P(r \mid x)}\!\left[D_{\mathrm{KL}}\!\left(P_T(\cdot \mid x, r) \;\|\; P_S(\cdot \mid x)\right)\right].
    \label{eq:opsd-obj}
\end{equation}
This forces a \textit{conditionally independent} parameterization ($P_S$ does not take $r$ as input) to match a \textit{conditionally dependent} target ($P_T$ depends on $r$), constituting a fundamentally ill-posed requirement.

\begin{theorem}[KL Decomposition]
\label{thm:kl-decomp}
The OPSD objective and the ideal objective satisfy the identity:
\begin{equation}
    \mathcal{L}_{\mathrm{OPSD}} \;=\; \mathcal{L}^* \;+\; I(Y_t;\, R \mid X,\, Y_{<t}),
    \label{eq:kl-decomp}
\end{equation}
where $I(Y_t; R \mid X, Y_{<t})$ denotes the conditional mutual information between the current token $Y_t$ and the privileged information $R$ under the teacher distribution.
\end{theorem}

The proof is deferred to Appendix~\ref{app:proof-kl-decomp}. The mutual information term quantifies the extent to which the teacher's token-level prediction depends on the privileged information that is, by construction, inaccessible to the student. Critically, $I(Y_t; R \mid X, Y_{<t})$ is \textit{independent of $\theta$}: it is entirely determined by the teacher's conditional distributions and $P(r \mid x)$. The student's optimization cannot eliminate this gap. Within the feasible set $\mathcal{F} = \{Q : Q(\cdot \mid x, y_{<t}) \text{ does not condition on } r\}$, the global optimum is $P_S^* = \bar{P}_T$, at which the residual loss equals $I(Y_t; R \mid X, Y_{<t}) > 0$, a strictly positive, irreducible lower bound that grows with the informativeness of the privileged signal.

This result provides a formal explanation for the KL stagnation observed in Figure~\ref{fig:pilot}(c): the OPSD divergence plateaus because the student quickly reaches the vicinity of $\bar{P}_T$, after which the residual $I(Y_t; R \mid X, Y_{<t}) > 0$ cannot be reduced through legitimate optimization. The dashed line in Figure~\ref{fig:pilot}(c), marking the OPSD KL at step~1, visually confirms that this loss floor is effectively reached within the earliest steps. More critically, this irreducible residual actively corrupts the optimization process. Under a biased objective, the optimizer continues to receive a non-vanishing loss signal and is driven to absorb harmful noise into the model parameters. As we will explain in the next section, this residual term directly contaminates the gradient direction, steering parameter updates away from genuine reasoning improvement and toward encoding spurious correlations between the input and the privileged information. Since the student's architecture precludes direct conditioning on $r$, the only available pathway is to encode statistical correlations between $x$ and $r$ in the parameters $\theta$, effectively learning a mapping $x \to r$. This is the mathematical origin of privileged information leakage. In contrast, OPD employs an external teacher whose predictions do not condition on privileged information inaccessible to the student, so the mutual information gap does not arise, and the KL divergence decreases steadily.


\subsubsection{Gradient Structure: The Mechanism of Leakage}
\label{sec:gradient-structure}

Theorem~\ref{thm:kl-decomp} establishes that $I(Y_t; R \mid X)$ is $\theta$-independent, which may suggest that it exerts no influence on gradients. We demonstrate that, while this holds for the \textit{expected} gradient, the \textit{per-sample} gradients carry a deviation term whose variance is directly governed by this mutual information.

\noindent \textbf{Benign expected gradient.}
Since $I(Y_t; R \mid X)$ does not depend on $\theta$, we have $\nabla_\theta \mathcal{L}_{\text{OPSD}} = \nabla_\theta \mathcal{L}^* = -\sum_{v} \bar{P}_T(v) \nabla_\theta \log P_S(v)$.
At the population level, the OPSD gradient is identical to that of the ideal marginal matching objective.

\noindent \textbf{Pathological per-sample gradients.}
In practice, optimization operates on concrete samples $(x, r)$:
\begin{equation}
    g(\theta; r) = -\sum_{v \in \mathcal{V}} P_T(v \mid r) \cdot \nabla_\theta \log P_S(v).
\end{equation}

\begin{proposition}[Per-Sample Gradient Decomposition]
\label{prop:grad-decomp}
For any specific realization of $r$, the per-sample gradient admits the decomposition:
\begin{equation}
    g(\theta; r) = \underbrace{-\sum_{v} \bar{P}_T(v)\, \nabla_\theta \log P_S(v)}_{g^*(\theta):\; \textit{marginal matching}} \;+\; \underbrace{-\sum_{v} \bigl[P_T(v \mid r) - \bar{P}_T(v)\bigr]\, \nabla_\theta \log P_S(v)}_{\delta(\theta;\, r):\; r\textit{-specific deviation}},
    \label{eq:grad-decomp}
\end{equation}
satisfying: (i) $\mathbb{E}_{r}[\delta(\theta; r)] = 0$, and (ii) $\mathbb{E}_{r}[\|\delta(\theta; r)\|^2] = \sum_{v} \mathrm{Var}_{r}[P_T(v \mid r)] \cdot \|\nabla_\theta \log P_S(v)\|^2$. The deviation vanishes identically when $I(Y_t; R \mid X) = 0$ and its variance increases monotonically with the mutual information.
\end{proposition}

The proof is provided in Appendix~\ref{app:proof-grad-decomp}. Property~(i) may suggest that the deviation is innocuous on average; however, any optimizer that computes gradients on individual samples or mini-batches, such as SGD and Adam \citep{kingma2017adammethodstochasticoptimization}, is inherently path-dependent. The zero-mean perturbations in nonlinear optimization do not necessarily cancel over the course of training.

\noindent \textbf{Two-phase training dynamics.}
The decomposition in Proposition~\ref{prop:grad-decomp} partitions the per-sample gradient into a beneficial component $g^*$ and a deviation component $\delta$. Their relative magnitudes evolve over training and give rise to two distinct regimes that precisely correspond to the empirical phenomena reported in \S\ref{sec:observations}.

Early in training, the student $P_S$ is far from the teacher's marginal $\bar{P}_T$, so the beneficial component dominates: $\|g^*(\theta)\| \gg \|\delta(\theta; r)\|$. In this regime, the gradient predominantly drives marginal matching, and the student rapidly acquires general reasoning capabilities. This corresponds to the steep rise in validation accuracy observed during the first 10\textendash 20 steps in Figure~\ref{fig:pilot}(b). As training progresses and $P_S$ approaches $\bar{P}_T$, the beneficial component $\|g^*(\theta)\|$ diminishes toward zero. The deviation component $\|\delta(\theta; r)\|$, however, remains bounded away from zero: its variance is governed by $I(Y_t; R \mid X)$, which is independent of $\theta$ and therefore does not decay with optimization progress. Parameter updates thus become increasingly dominated by $\delta$, and the path-dependent accumulation of these perturbations drives the model toward regions of parameter space that encode $x \to r$ correlations, triggering a self-reinforcing degradation. This transition marks the onset of the performance decline in Figure~\ref{fig:pilot}(b) and the monotonically increasing leakage counts in Figure~\ref{fig:pilot}(a).

\noindent \textbf{Leakage bandwidth: controlled experiments.}
The gradient decomposition not only explains the failure of standard OPSD, but also makes a precise prediction: \textit{any} variant in which the teacher's privileged evaluation $P_T(\cdot \mid r)$ enters the gradient direction will suffer from leakage, regardless of how the distillation target is compressed. To test this prediction, we design two ablated variants alongside standard OPSD: (i) \textit{Teacher's Top-1}, which retains only the teacher's most probable token $\arg\max_v P_T(v \mid r)$ as the target, and (ii) \textit{Student's Top-1}, which restricts the target support to the student's most probable token $\arg\max_v P_S(v)$. We report leakage counts and validation accuracy for all three variants over 100 training steps in Figure~\ref{fig:pilot}(a,b).

All three variants confirm the prediction: leakage increases in every case. The gradient framework explains both the universality and the ordering of severity through the concept of \textit{leakage bandwidth}, which we define as the effective number of token positions at which $r$-specific information enters the gradient direction. Full OPSD operates over the entire vocabulary $\mathcal{V}$: the teacher's privileged preference $P_T(v \mid r)$ weights the gradient contribution of every token, yielding the widest bandwidth. Teacher's Top-1 collapses the target to a single token $\arg\max_v P_T(v \mid r)$ that is entirely determined by $r$, narrowing the bandwidth but producing the most concentrated injection of privileged information, which explains why it exhibits the most severe leakage. Student's Top-1 restricts the target support to $\arg\max_v P_S(v)$, yielding the narrowest bandwidth; yet the gradient weight $P_T(v_S^* \mid r)/P_S(v_S^*)$ at the selected token remains a function of $r$, so leakage persists, albeit at the lowest rate. In all three cases, the directional dependence of the gradient on $r$ is irreducible (proof in Appendix~\ref{app:proof-unified-pilot}), explaining the universal occurrence of leakage observed in Figure~\ref{fig:pilot}(a). The performance degradation in Figure~\ref{fig:pilot}(b) follows directly: stronger leakage accelerates the transition from the first training phase to the second and produces more rapid decline.

We further analyze the implications of shared-parameter coupling in Appendix~\ref{app:trilemma}, where we show that the interaction between gradient-level leakage and teacher drift under shared parameters gives rise to an impossibility trilemma: objective stability, sustained improvement, and leakage-free training cannot hold simultaneously under any parameter management strategy.
 

\section{RLSD: Self-Distillation as RLVR's Wingman}
\label{sec:method}
 
The preceding analysis identifies a precise root cause: distribution matching fails because privileged information enters the \textit{gradient direction}, contaminating the optimization trajectory. Yet the same quantity at the heart of this failure, the evidence ratio $P_T(y_t)/P_S(y_t)$, also carries a useful signal: it measures how much the privileged information revises the model's belief about each token. The challenge, then, is not to discard this signal, but to change \textit{how} it is used. We propose RLVR with Self-Distillation (RLSD), which repurposes the teacher's role entirely: instead of serving as a generative target for distribution matching, the discrepancy between $P_T$ and $P_S$ is used as a token-level credit assignment signal within the policy gradient framework. \textbf{The privileged information influences only how much credit each token receives, not which tokens are reinforced or penalized, nor the direction of parameter updates}.

\subsection{From Distribution Matching to Credit Assignment}
\label{sec:rlsd}
 
\noindent \textbf{Step 1: Privileged information gain.}
Given a student-sampled trajectory $y = (y_1, \dots, y_T)$, we compute the log-probability of each token under both the student context ($x$ only) and the teacher context ($x$ and $r$), and define the \textit{privileged information gain} at each position:
\begin{equation}
    \Delta_t = \texttt{sg}\!\left(\log P_T(y_t) - \log P_S(y_t)\right),
    \label{eq:delta}
\end{equation}
where $\texttt{sg}$ denotes the stop-gradient operator. Since teacher and student share the same model, $\Delta_t$ isolates the \textit{marginal contribution of the privileged information} $r$ to the prediction of $y_t$. A large positive $\Delta_t$ indicates that $r$ strongly supports this token; a negative value indicates that $r$ disfavors it. Crucially, $\Delta_t$ provides a dense, token-level signal that naturally reflects how much each token is informed by the privileged information, making it a principled and lightweight basis for fine-grained credit assignment within a trajectory. The stop-gradient ensures that $\Delta_t$ serves purely as a weighting signal and does not introduce auxiliary gradient pathways.

\begin{figure*}[t]
  \centerline{\includegraphics[scale=0.44]{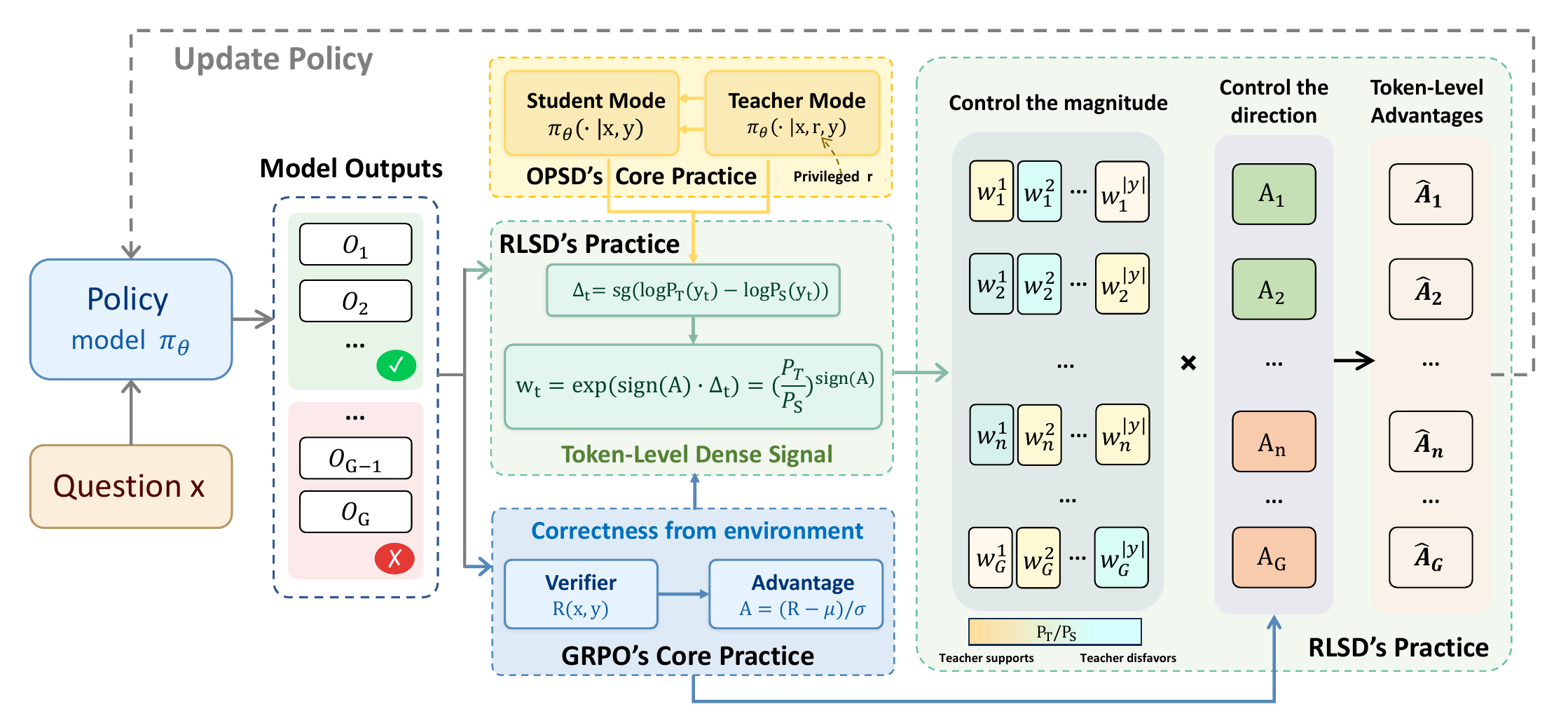}}
  \vspace{-0.1cm}
  \caption{An overview of our RLSD method. }
  \label{fig:method}
\end{figure*} 
 
\noindent \textbf{Step 2: Direction-aware evidence reweighting.}
We construct a per-token weight from the privileged information gain, modulated by the sign of the sequence-level advantage:
\begin{equation}
    w_t = \exp\!\left(\mathrm{sign}(A) \cdot \Delta_t\right) = \left(\frac{P_T(y_t)}{P_S(y_t)}\right)^{\mathrm{sign}(A)}.
    \label{eq:weight}
\end{equation}
 
This formulation admits a natural Bayesian interpretation. $P_S(y_t)$ represents the model's \textit{prior} assessment of token $y_t$, based solely on the question $x$. $P_T(y_t)$ represents the \textit{posterior} assessment after observing the privileged information $r$. The ratio $P_T(y_t)/P_S(y_t)$ is therefore an \textit{evidence ratio}: the factor by which the privileged information revises the model's belief about each token. Under mild modeling assumptions, this ratio can be shown to equal the Bayesian belief update $P(r \mid x, y_{\leq t}) / P(r \mid x, y_{<t})$, i.e., the degree to which generating $y_t$ increases the posterior probability that the privileged information $r$ is consistent with the trajectory (Theorem~\ref{thm:bayesian} in Appendix~\ref{app:bayesian}).
 
The $\mathrm{sign}(A)$ exponent implements direction-aware credit assignment. When $A > 0$, $w_t = P_T/P_S$: tokens that the privileged information supports receive larger weights, concentrating positive credit on the tokens most aligned with the correct reasoning trace. When $A < 0$, $w_t = P_S/P_T$: the ratio is inverted, so tokens that the privileged information \textit{disfavors} bear greater blame, while tokens it supports receive attenuated punishment. Since $\exp(\cdot) > 0$ for all inputs, the weight is strictly positive, guaranteeing that the sign of the token-level advantage is never flipped by the reweighting. The environment reward retains exclusive authority over whether a trajectory is reinforced or penalized; the teacher only modulates the relative magnitude across tokens within a trajectory.
 
This design parallels the importance sampling ratio $\pi_\theta / \pi_{\mathrm{old}}$ used in GRPO for policy updates. GRPO uses the ratio between current and old policies to control the \textit{step size} of updates; RLSD uses the ratio between posterior and prior to control the \textit{credit distribution} across tokens. Both are importance ratios operating within the same policy gradient framework, yielding a structurally unified formulation.

\noindent \textbf{Step 3: Clipped credit assignment.}
Following the design philosophy of the clipped surrogate objective in PPO~\citep{schulman2017proximalpolicyoptimizationalgorithmsppo} and GRPO, we clip the evidence weights to bound the maximum influence of any single token:
\begin{equation}
    \hat{A}_t = A \cdot \mathrm{clip}\!\left(w_t,\; 1 - \epsilon_w,\; 1 + \epsilon_w\right),
    \label{eq:final}
\end{equation}
where $\epsilon_w$ bounds the per-token credit deviation. The clipping in Eq.~\eqref{eq:final} serves an analogous role to the importance ratio clipping in GRPO: while GRPO clips the policy update step size, RLSD clips the credit redistribution magnitude. Both mechanisms act as trust-region constraints that stabilize training.  In practice, to avoid an abrupt transition at the start of training, we linearly interpolate between the uniform advantage and the reweighted advantage over the training steps using $\lambda \in [0, 1]$, gradually shifting to the uniform advantage. The final objective of RLSD is as follows:
\begin{equation}
    \mathcal{L}_{\text{RLSD}}(\theta) = \mathbb{E}\left\{\frac{1}{G}\sum_{i=1}^{G}\frac{1}{|y^{(i)}|}\sum_{t=1}^{|y^{(i)}|} \min\! \left[ w_tA^{(i)} ,\; \text{clip}\!\left(w_t, 1-\epsilon_w, 1+\epsilon_w\right)A^{(i)} \right]\right\},
    \label{eq:rlsd}
\end{equation}

\subsection{Integration with GRPO}
\label{sec:integration}
 
The modified advantage $\hat{A}_t$ is a drop-in replacement for the uniform advantage in the standard GRPO objective. The complete training procedure is summarized in Algorithm~\ref{alg:rlsd}.

No auxiliary distillation loss is introduced; the only modification to the standard GRPO pipeline is the internal redistribution of credit within each trajectory. The additional computational cost amounts to one forward pass per response to obtain the teacher logits, which is negligible relative to the rollout generation that dominates wall-clock time.

\begin{algorithm}[t]
\caption{RLSD: Reinforcement Learning with Self-Distillation}
\label{alg:rlsd}
\begin{algorithmic}[1]
\Require Policy $\pi_\theta$, dataset $\mathcal{S} = \{(x_i, r_i)\}_{i=1}^N$, verifier $R(\cdot, \cdot)$, group size $G$, mixing coefficient $\lambda$, clip bounds $\epsilon_w$
\For{each training iteration}
    \State Sample a batch of questions $\{x\}$ from $\mathcal{S}$
    \For{each question $x$ with privileged information $r$}
        \State \textcolor{gray}{\textit{// Step 1: On-policy rollout}}
        \State Sample $G$ responses $\{y^{(1)}, \dots, y^{(G)}\} \sim \pi_\theta(\cdot \mid x)$
        \State \textcolor{gray}{\textit{// Step 2: Sequence-level advantage from environment}}
        \For{$i = 1, \dots, G$}
            \State Obtain reward $R(x, y^{(i)}) \in \{0, 1\}$ from the verifier
        \EndFor
        \State Compute $A^{(i)} = \frac{R(x, y^{(i)}) - \mu_G}{\sigma_G}$ \Comment{Group-relative advantage}
        \State \textcolor{gray}{\textit{// Step 3: Token-level credit assignment via self-distillation}}
        \For{$i = 1, \dots, G$}
            \State Compute teacher logits via forward pass with $(x, r, y^{(i)})$ \Comment{Single extra forward pass}
            \For{$t = 1, \dots, |y^{(i)}|$}
                \State $\Delta_t \gets \texttt{sg}\!\left(\log \pi_\theta(y_t^{(i)} \mid x, r, y_{<t}^{(i)}) - \log \pi_\theta(y_t^{(i)} \mid x, y_{<t}^{(i)})\right)$ 
                \State $w_t \gets \exp\!\left(\text{sign}(A^{(i)}) \cdot \Delta_t\right)$ 
                \State $\hat{A}_t^{(i)} \gets A^{(i)} \cdot \left((1 - \lambda) + \lambda \cdot \text{clip}(w_t,\, 1 - \epsilon_w,\, 1 + \epsilon_w)\right)$ 
            \EndFor
        \EndFor
    \EndFor
    \State \textcolor{gray}{\textit{// Step 4: Policy update}}
    \State Update $\theta$ by maximizing $\mathcal{L}_{\text{RLSD}}(\theta) = \frac{1}{G} \sum_{i=1}^{G} \frac{1}{|y^{(i)}|} \sum_{t=1}^{|y^{(i)}|} \hat{A}_t^{(i)}$
\EndFor
\end{algorithmic}
\end{algorithm}

\subsection{A Unified Token-Level Advantage Perspective}
\label{sec:unified}
 
To situate RLSD within the broader landscape, we observe that GRPO, on-policy self-distillation, and RLSD can all be expressed as instances of a single policy gradient template:
\begin{equation}
    \Delta\theta \;\propto\; \mathbb{E}_{y \sim P_S(\cdot|x)}\!\left[\sum_{t=1}^{|y|} \hat{A}_t \;\nabla_\theta \log P_S(y_t \mid x, y_{<t})\right],
    \label{eq:unified}
\end{equation}
where the methods differ solely in how they define the token-level advantage $\hat{A}_t$. 
 
\textit{GRPO} assigns a uniform advantage across all tokens: $\hat{A}_t = A$, where $A$ is the sequence-level advantage from the verifier. This ensures that the optimization direction is fully grounded in the environment reward, but provides no token-level discrimination: every token within a trajectory receives identical credit regardless of its contribution to the final answer.
 
\textit{On-policy self-distillation} methods replace the environment reward with a dense teacher signal. The specific form of $\hat{A}_t$ depends on the divergence measure: minimizing the reverse KL $D_{\mathrm{KL}}(P_S \| P_T)$ via the log-derivative trick yields $\hat{A}_t = \Delta_t = \log P_T(y_t) - \log P_S(y_t)$.
The environment reward $R(x,y)$ is entirely absent from $\hat{A}_t$: even when a trajectory produces a wrong answer ($A < 0$), tokens favored by the teacher ($\Delta_t > 0$) still receive positive advantage, decoupling the optimization direction from the verifiable correctness signal.
 
\textit{RLSD} resolves this tension by combining both sources of information. Its advantage (Eq.~\ref{eq:final}) uses the environment reward to determine the \textit{direction} (sign) of each token's update, while using the teacher's privileged assessment to determine the \textit{magnitude} (relative credit) within a trajectory. We provide a formal proof that these properties render RLSD structurally immune to privileged information leakage in Appendix~\ref{app:leakage-free}, and verify that RLSD simultaneously satisfies all three desiderata of the impossibility trilemma (Appendix~\ref{app:trilemma}).
 

\section{Experiment}
\label{sec:exp}

\subsection{Experimental Setup}

\noindent \textbf{Training Data and Benchmarks.}
We train our models on MMFineReason-123K~\citep{lin2026mmfinereasonclosingmultimodalreasoning}, a challenging subset derived from the MMFineReason-1.8M corpus via difficulty-based filtering. Specifically, inference is performed on each training sample using Qwen3-VL-4B-Thinking with four independent rollouts, and only samples where the model fails on all attempts are retained. This conservative criterion discards trivial examples and concentrates training signal on challenging problems, yielding faster convergence and more efficient use of computational resources.

We evaluate on five multimodal reasoning benchmarks spanning diverse mathematical and general reasoning capabilities. \textbf{MMMU}~\citep{mmmu} is a massive multi-discipline benchmark covering college-level subjects across science, engineering, and humanities, requiring both perception and domain knowledge. \textbf{MathVista}~\citep{mathvista} assesses mathematical reasoning grounded in visual contexts. \textbf{MathVision}~\citep{mathvision} extends mathematical reasoning evaluation to more complex and competition-level visual problems. \textbf{ZeroBench}~\citep{zerobench} is a challenging benchmark specifically designed to be unsolvable by current frontier models, providing a stress test for reasoning robustness. \textbf{WeMath}~\citep{wemath} evaluates fine-grained mathematical problem-solving abilities with structured difficulty levels. Together, these benchmarks provide a comprehensive and rigorous assessment of both the breadth and depth of multimodal reasoning performance. For all benchmarks, we report accuracy (\%) as the evaluation metric.

\noindent \textbf{Models and Baselines.}
We conduct experiments on \textbf{Qwen3-VL-8B-Instruct}~\citep{qwen3vl} as our base model. We compare RLSD against the following baselines. \textbf{GRPO}~\citep{grpo} is a standard RLVR method that estimates token-level advantages from sparse scalar outcome rewards at the sequence level, serving as our primary RL baseline. \textbf{OPSD}~\citep{zhao2026opsd} is an on-policy self-distillation method where a single model acts as both teacher and student, with the teacher conditioned on privileged information (e.g., verified reasoning traces) to provide dense per-token supervision. \textbf{SDPO}~\citep{hbotter2026reinforcementlearningselfdistillation_sdpo} extends self-distillation to the reinforcement learning with rich feedback setting, treating the current model conditioned on environment feedback as a self-teacher to derive dense credit assignment signals. \textbf{GRPO+OPSD} is a straightforward combination baseline inspired by the MOPD approach proposed in MIMO-v2-Flash~\citep{mimo2025flash}, which jointly optimizes reinforcement learning and distillation objectives. Specifically, we combine the GRPO loss with a distillation KL divergence loss via linear interpolation with a tuned weighting coefficient. Unlike MOPD, which relies on a separately trained expert model for distillation, we adopt the OPSD formulation by using the model's own rollout outputs as the distillation target, thereby eliminating the need for an external expert while still leveraging dense self-distillation signals alongside sparse environment rewards. We also include the \textbf{Base LLM} (i.e., Qwen3-VL-8B-Instruct without any post-training) as a reference to quantify the overall gains from post-training.


\noindent \textbf{Implementation Details.}
We implement our method based on the VERL~\citep{verl} and EasyR1~\citep{easyr1} frameworks. 
During training, the model's maximum context
length is set to 8192, with a maximum prompt length of 4096 and a maximum response length of
4096.
The same prompt and response length settings are used consistently during both training and evaluation. For GRPO, GRPO+OPSD, and RLSD, the learning rate is fixed at $1 \times 10^{-6}$; for OPSD and SDPO, the learning rate is set to $1 \times 10^{-5}$ following their original implementations. The training batch size is set to 256, and for each prompt, we sample 8 rollouts with a sampling temperature of 1.0. For all GRPO-based methods, we use clipping thresholds of $\epsilon_{\text{low}} = 0.2$ and $\epsilon_{\text{high}} = 0.28$, and omit both the KL penalty loss and entropy regularization loss from the objective function. The hyperparameter $\lambda$ in our RLSD is initialized at $0.5$ and linearly decayed to 0 over the first 50 training steps, and $\epsilon_w$ is set to $0.2$. To maintain a stable self-distillation signal, the teacher model parameters are synchronized with the student model every 10 training steps and kept frozen in between. Regarding the privileged information required by different methods: OPSD requires verified reasoning traces, for which we use reasoning trajectories distilled from Qwen3-VL-235B-A22B-Thinking and verified as correct within the MMFineReason-123K dataset; SDPO follows its original setting in the paper, using a successful previous rollout as privileged context; RLSD requires \textbf{only the final ground-truth answer} without any reasoning trace, making it the least demanding in terms of privileged information. All experiments are conducted on 4 compute nodes, each equipped with 8 NVIDIA H200 140GB GPUs.

\subsection{Experimental Results}

\subsubsection{Main Results}
\label{sec:main-results}

\begin{table*}[t]
    \caption{%
        Multimodal reasoning results on the Qwen3-VL-8B-Instruct model.
        \textbf{Bold} indicates the best result among all models.
        Avg.\ is computed over the five selected benchmarks.
    }
    \label{tab:multimodel}
    \centering
    \begin{tabular}{lccccc|c}
        \toprule
        Method
            & MMMU & MathVista & MathVision & ZeroBench
            & Wemath
            & Avg. \\
        \midrule
        Base LLM
            & 62.44 & 73.80  
            & 47.37 & 19.76 & 54.10
            & 51.49 \\
        GRPO
            & 65.11 & 76.20
            & 48.82 & 22.60 & 56.57
            & 53.86 \\
        OPSD
            & 63.82 & 75.10
            & 47.53 & 21.06 & 54.95
            & 52.49 \\
        SDPO
            & 65.11 & 74.00 
            & 47.27 & \textbf{25.15} & 52.19 
            & 52.74 \\
        GRPO+OPSD
            & 63.22 & 75.90
            & 48.52 & 22.16 & 54.76
            & 52.91 \\
            \rowcolor[rgb]{0.87,0.94,1}
        RLSD \textit{(Ours)}
            & \textbf{67.22} & \textbf{78.10} & \textbf{52.73} & {24.85} & \textbf{58.00}
            & \textbf{56.18} \\
        \bottomrule
    \end{tabular}
\end{table*}
\begin{figure*}[t]
  \centerline{\includegraphics[scale=0.45]{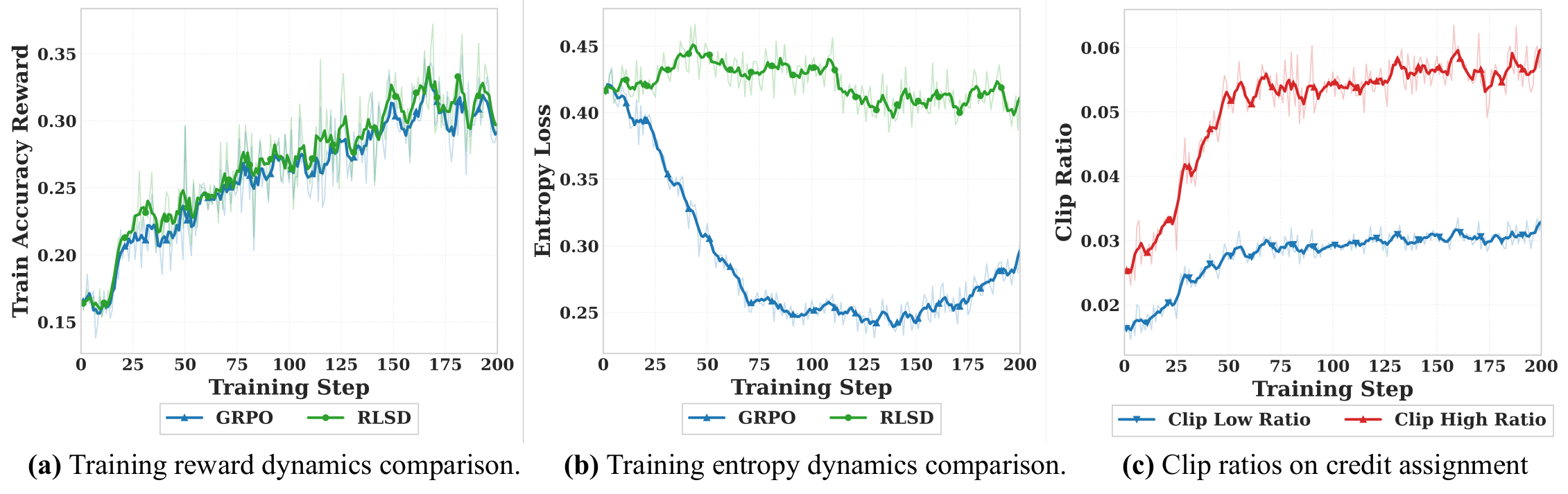}}
  \vspace{-0.1cm}
  \caption{Training dynamics on the multimodel reasoning tasks. }
  \label{fig-dynamics}
\end{figure*} 

Table~\ref{tab:multimodel} presents the evaluation results across five multimodal reasoning benchmarks. RLSD achieves the highest average accuracy, outperforming the Base LLM by 4.69\% and GRPO by 2.32\% on average under the 4K setting. Compared to GRPO, RLSD benefits from dense token-level credit assignment, yielding notable gains on challenging mathematical datasets such as MathVista (+1.9\%) and MathVision (+3.91\%), where fine-grained discrimination of reasoning steps matters most. RLSD also consistently outperforms the self-distillation baselines (OPSD and SDPO) in average accuracy. This empirical gap is consistent with our theoretical diagnosis in Section~\ref{sec:ill-posed}: while OPSD suffers from privileged information leakage due to distribution matching under information asymmetry, RLSD extracts the dense teacher signal while anchoring the update direction to the verifier reward, thereby producing more robust reasoning improvements. Moreover, RLSD also outperforms the additive fusion methods GRPO+OPSD by a margin of 3.27 points. Linearly combining bounded rewards with unbounded, high-variance kl losses causes severe scale mismatch, forcing a suboptimal trade-off that often destabilizes training. In contrast, RLSD effectively circumvents these pitfalls through multiplicative modulation. By exponentiating the teacher's signal into a dynamically bounded relative multiplier, RLSD mathematically guarantees strict sign preservation.

\subsubsection{Training Dynamics}
\label{sec:training-dynamics}

Figure~\ref{fig-dynamics} illustrates the training dynamics over 200 optimization steps. As shown in Figure~\ref{fig-dynamics}(a), RLSD demonstrates a steeper initial ascent and converges to a higher accuracy reward ceiling compared to standard GRPO, while avoiding the late-stage performance collapse observed in OPSD. Figure~\ref{fig-dynamics}(b) reveals that GRPO suffers from rapid entropy collapse due to its uniform sequence-level reward, whereas RLSD maintains a consistently higher entropy level by selectively strengthening critical reasoning tokens without uniformly suppressing alternatives at every position. Finally, Figure~\ref{fig-dynamics}(c) confirms that the clipped credit assignment mechanism is actively engaged, with clip ratios stabilizing around 3\%--6\%, successfully bounding the teacher's per-token influence and acting as a trust-region constraint analogous to the importance ratio clipping in PPO/GRPO.

\subsubsection{Case Study}
\label{sec:case-study}


\begin{figure*}[t]
  \centerline{\includegraphics[scale=0.5]{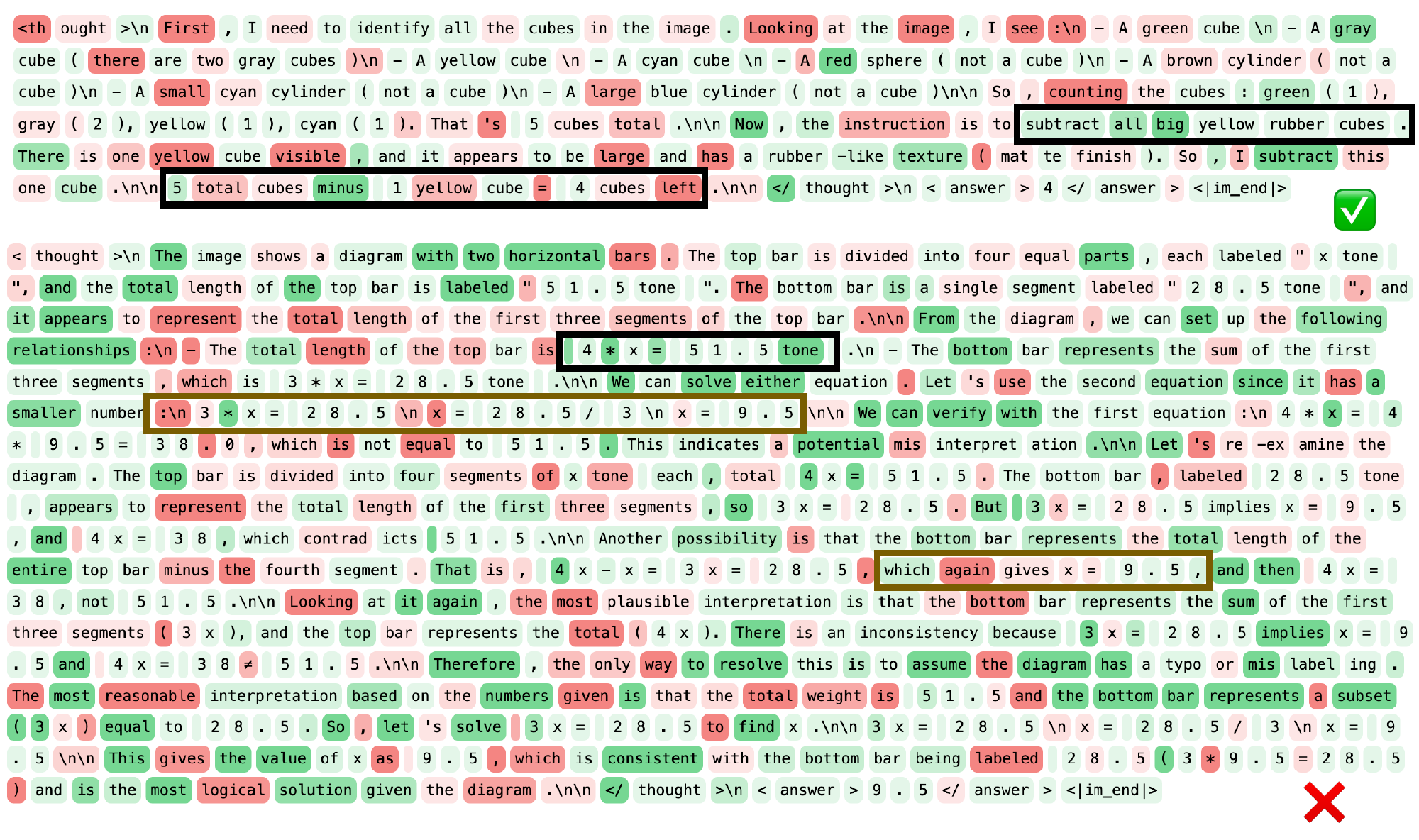}}
  \vspace{-0.1cm}
  \caption{Token-level credit heatmaps under RLSD for two multimodal reasoning examples. Top: in a correct trajectory, larger credit is concentrated on the decisive counting and subtraction steps. Bottom: in an incorrect trajectory, the strongest blame is assigned to the misread bar-model relation and the derived wrong answer. Green denotes larger token-level credit magnitude and red denotes smaller token-level credit magnitude.}
  \label{fig:case-study-exp}
\end{figure*}

Figure~\ref{fig:case-study-exp} provides a qualitative view of how RLSD redistributes sequence-level reward to the token level. In the correct cube-counting example, RLSD assigns larger credit to the tokens that actually determine correctness, such as identifying the relevant yellow cube and executing the final subtraction, while down-weighting generic narration like ``Looking at the image, I see \ldots''. In contrast, in the incorrect bar-model example, RLSD concentrates the strongest blame on the misread relation ``3x = 28.5'' and the derived wrong answer ``x = 9.5'', whereas neutral setup tokens receive comparatively smaller penalties.

This behavior is consistent with the design goal of RLSD. The environment reward still determines whether a trajectory is reinforced or penalized, but the privileged teacher modulates the relative magnitude of token-level credit according to each token's contribution to the final outcome. The resulting update pattern is neither uniform like GRPO nor distribution matching like OPSD; instead, it acts as targeted token-level credit assignment anchored to verifier-grounded correctness.

\section{Related Work}

\subsection{Credit Assignment in RLVR}
Reinforcement learning with verifiable rewards (RLVR) methods such as GRPO~\citep{grpo} have become a standard paradigm for post-training reasoning models because they optimize against automatically checkable outcomes without human preference labels~\citep{grpo,deepseekr1}. Their central limitation is that the verifier provides a sequence-level signal, so every token in a rollout inherits the same advantage regardless of whether it corresponds to a pivotal deduction or a stylistic filler. This makes long-horizon reasoning fundamentally a credit assignment problem. A common response is to introduce process reward models (PRMs) or step-level value estimators that score intermediate reasoning steps~\citep{lightman2024lets,wang2024math,yang2025testtimepromptintervention,luo2024improve,chen2024step,zhang2024generative,yang2025dynamicearlyexitreasoning,dai2025sgrpoearlyexitreinforcement,cui2025process}. Some of these methods rely on costly human step annotations, while others replace them with automated or implicit supervision; in either case, they require auxiliary reward modeling and extra computation beyond the base policy.

More recent work seeks fine-grained credit assignment while staying strictly within a verifier-only RLVR pipeline. These methods reshape token-level updates using model-internal proxies such as entropy, uncertainty, key-token statistics, attention dynamics, or outcome sensitivity~\citep{cheng2026reasoning,wangbeyond,chen2025seed,sun2025ktaemodelfreealgorithmkeytokens,xie2025unlocking,li2025attention,chen2025beyond,li2026outcome}. While they show that uniform sequence-level advantages are unnecessarily coarse, they estimate token importance entirely from intrinsic heuristics. Our method is complementary in motivation yet fundamentally different in mechanism: rather than relying on heuristic proxies, we extract rigorous token-level assessments via self-distillation under privileged context, while keeping the update direction anchored to the verifier reward.

\subsection{On-Policy Distillation}
Distillation has also been explored as an alternative source of dense token-level supervision for reasoning models. In on-policy distillation (OPD), a stronger teacher evaluates the student's own rollouts and provides token-level targets along on-policy trajectories~\citep{agarwal2024gkd,lu2025onpolicy,fu2026revisitingonpolicydistillationempirical}; recent reports suggest that strong teacher-guided distillation can even rival or complement RL-based post-training~\citep{mimo2025flash}. However, standard OPD requires maintaining an external, typically larger teacher throughout training. To eliminate this overhead, recent self-distillation methods let the same model act as both student and teacher, with the teacher conditioned on privileged information such as reference solutions or verifier feedback~\citep{zhao2026opsd,hbotter2026reinforcementlearningselfdistillation_sdpo}. Closely related on-policy distillation variants have also been explored for continual learning from demonstrations, context internalization, and reasoning compression, where the teacher is conditioned on demonstrations, auxiliary context, or conciseness instructions~\citep{shenfeld2026selfdistillation,ye2026policy,sang2026policy,penaloza2026privileged}. A recent concurrent work, TRRD \citep{zhang2026reinforcementawareknowledgedistillationllm}, also identifies the conflict between additive KL penalties and reward maximization, proposing to inject teacher probabilities into the policy importance ratio. However, TRRD operates on the trust region anchor while RLSD directly modulates the advantage magnitude via a Bayesian evidence ratio.

The shared design choice across these methods is to retain a distribution-matching objective between a privileged or context-conditioned teacher and a less-informed student. Our work starts from a similar setting but departs from that paradigm. We show that when the teacher's prediction depends on latent information unavailable to the student, matching the teacher distribution becomes structurally ill-posed and can induce privileged information leakage and unstable long-run training. RLSD therefore does not use the teacher as a generative target; instead, it repurposes the privileged discrepancy only as a scalar multiplier for credit assignment, decoupling dense token-level magnitude from the optimization direction.

\section{Conclusion}


In this work, we identified the fundamental limitation of on-policy self-distillation (OPSD): the information asymmetry between teacher and student renders the distribution-matching objective structurally ill-posed, inducing an irreducible mutual information gap that drives privileged information leakage and progressive performance degradation. We formalized this deficiency at both the objective level and the gradient level, explaining the empirically observed pattern of early gains followed by systematic collapse.
Based on this analysis, we proposed RLSD, a new training paradigm that moves beyond distribution matching between the prior and posterior: instead of forcing the student to imitate the teacher's conditional distribution, RLSD repurposes the discrepancy between them to provide fine-grained control over per-token update magnitudes, while anchoring the update directions to the environment reward. This design simultaneously inherits the rich, token-level supervision of OPSD and the reliable environmental grounding of RLVR, unifying the strengths of both paradigms within a single framework.

\section{Limitations and Future Work}
This paper focuses primarily on the theoretical analysis of OPSD's structural limitations and the motivation and validation of the RLSD paradigm. To enable faster release, this version provides limited experiments, validating around multimodal reasoning scenarios. However, we have preliminarily validated RLSD across a broader range of settings, including pure text reasoning, video understanding, and additional model families beyond the Qwen series, and observed consistent gains. We will include these in the forthcoming version.

\bibliographystyle{unsrtnat}
\bibliography{neurips_2026}

\begin{thebibliography}{42}
\providecommand{\natexlab}[1]{#1}
\providecommand{\url}[1]{\texttt{#1}}
\expandafter\ifx\csname urlstyle\endcsname\relax
  \providecommand{\doi}[1]{doi: #1}\else
  \providecommand{\doi}{doi: \begingroup \urlstyle{rm}\Url}\fi

\bibitem[Shao et~al.(2024)Shao, Wang, Zhu, Xu, Song, Bi, Zhang, Zhang, Li, Wu, and Guo]{grpo}
Zhihong Shao, Peiyi Wang, Qihao Zhu, Runxin Xu, Junxiao Song, Xiao Bi, Haowei Zhang, Mingchuan Zhang, Y.~K. Li, Y.~Wu, and Daya Guo.
\newblock Deepseekmath: Pushing the limits of mathematical reasoning in open language models, 2024.
\newblock URL \url{https://arxiv.org/abs/2402.03300}.

\bibitem[DeepSeek-AI et~al.(2025)DeepSeek-AI, Guo, Yang, Zhang, Song, Zhang, Xu, Zhu, Ma, Wang, Bi, Zhang, Yu, Wu, Wu, Gou, Shao, Li, Gao, Liu, Xue, Wang, Wu, Feng, Lu, Zhao, Deng, Zhang, Ruan, Dai, Chen, Ji, Li, Lin, Dai, Luo, Hao, Chen, Li, Zhang, Bao, Xu, Wang, Ding, Xin, Gao, Qu, Li, Guo, Li, Wang, Chen, Yuan, Qiu, Li, Cai, Ni, Liang, Chen, Dong, Hu, Gao, Guan, Huang, Yu, Wang, Zhang, Zhao, Wang, Zhang, Xu, Xia, Zhang, Zhang, Tang, Li, Wang, Li, Tian, Huang, Zhang, Wang, Chen, Du, Ge, Zhang, Pan, Wang, Chen, Jin, Chen, Lu, Zhou, Chen, Ye, Wang, Yu, Zhou, Pan, Li, Zhou, Wu, Ye, Yun, Pei, Sun, Wang, Zeng, Zhao, Liu, Liang, Gao, Yu, Zhang, Xiao, An, Liu, Wang, Chen, Nie, Cheng, Liu, Xie, Liu, Yang, Li, Su, Lin, Li, Jin, Shen, Chen, Sun, Wang, Song, Zhou, Wang, Shan, Li, Wang, Wei, Zhang, Xu, Li, Zhao, Sun, Wang, Yu, Zhang, Shi, Xiong, He, Piao, Wang, Tan, Ma, Liu, Guo, Ou, Wang, Gong, Zou, He, Xiong, Luo, You, Liu, Zhou, Zhu, Xu, Huang, Li, Zheng, Zhu, Ma, Tang, Zha, Yan, Ren, Ren, Sha, Fu, Xu, Xie, Zhang,
  Hao, Ma, Yan, Wu, Gu, Zhu, Liu, Li, Xie, Song, Pan, Huang, Xu, Zhang, and Zhang]{deepseekr1}
DeepSeek-AI, Daya Guo, Dejian Yang, Haowei Zhang, Junxiao Song, Ruoyu Zhang, Runxin Xu, Qihao Zhu, Shirong Ma, Peiyi Wang, Xiao Bi, Xiaokang Zhang, Xingkai Yu, Yu~Wu, Z.~F. Wu, Zhibin Gou, Zhihong Shao, Zhuoshu Li, Ziyi Gao, Aixin Liu, Bing Xue, Bingxuan Wang, Bochao Wu, Bei Feng, Chengda Lu, Chenggang Zhao, Chengqi Deng, Chenyu Zhang, Chong Ruan, Damai Dai, Deli Chen, Dongjie Ji, Erhang Li, Fangyun Lin, Fucong Dai, Fuli Luo, Guangbo Hao, Guanting Chen, Guowei Li, H.~Zhang, Han Bao, Hanwei Xu, Haocheng Wang, Honghui Ding, Huajian Xin, Huazuo Gao, Hui Qu, Hui Li, Jianzhong Guo, Jiashi Li, Jiawei Wang, Jingchang Chen, Jingyang Yuan, Junjie Qiu, Junlong Li, J.~L. Cai, Jiaqi Ni, Jian Liang, Jin Chen, Kai Dong, Kai Hu, Kaige Gao, Kang Guan, Kexin Huang, Kuai Yu, Lean Wang, Lecong Zhang, Liang Zhao, Litong Wang, Liyue Zhang, Lei Xu, Leyi Xia, Mingchuan Zhang, Minghua Zhang, Minghui Tang, Meng Li, Miaojun Wang, Mingming Li, Ning Tian, Panpan Huang, Peng Zhang, Qiancheng Wang, Qinyu Chen, Qiushi Du, Ruiqi Ge, Ruisong
  Zhang, Ruizhe Pan, Runji Wang, R.~J. Chen, R.~L. Jin, Ruyi Chen, Shanghao Lu, Shangyan Zhou, Shanhuang Chen, Shengfeng Ye, Shiyu Wang, Shuiping Yu, Shunfeng Zhou, Shuting Pan, S.~S. Li, Shuang Zhou, Shaoqing Wu, Shengfeng Ye, Tao Yun, Tian Pei, Tianyu Sun, T.~Wang, Wangding Zeng, Wanjia Zhao, Wen Liu, Wenfeng Liang, Wenjun Gao, Wenqin Yu, Wentao Zhang, W.~L. Xiao, Wei An, Xiaodong Liu, Xiaohan Wang, Xiaokang Chen, Xiaotao Nie, Xin Cheng, Xin Liu, Xin Xie, Xingchao Liu, Xinyu Yang, Xinyuan Li, Xuecheng Su, Xuheng Lin, X.~Q. Li, Xiangyue Jin, Xiaojin Shen, Xiaosha Chen, Xiaowen Sun, Xiaoxiang Wang, Xinnan Song, Xinyi Zhou, Xianzu Wang, Xinxia Shan, Y.~K. Li, Y.~Q. Wang, Y.~X. Wei, Yang Zhang, Yanhong Xu, Yao Li, Yao Zhao, Yaofeng Sun, Yaohui Wang, Yi~Yu, Yichao Zhang, Yifan Shi, Yiliang Xiong, Ying He, Yishi Piao, Yisong Wang, Yixuan Tan, Yiyang Ma, Yiyuan Liu, Yongqiang Guo, Yuan Ou, Yuduan Wang, Yue Gong, Yuheng Zou, Yujia He, Yunfan Xiong, Yuxiang Luo, Yuxiang You, Yuxuan Liu, Yuyang Zhou, Y.~X. Zhu,
  Yanhong Xu, Yanping Huang, Yaohui Li, Yi~Zheng, Yuchen Zhu, Yunxian Ma, Ying Tang, Yukun Zha, Yuting Yan, Z.~Z. Ren, Zehui Ren, Zhangli Sha, Zhe Fu, Zhean Xu, Zhenda Xie, Zhengyan Zhang, Zhewen Hao, Zhicheng Ma, Zhigang Yan, Zhiyu Wu, Zihui Gu, Zijia Zhu, Zijun Liu, Zilin Li, Ziwei Xie, Ziyang Song, Zizheng Pan, Zhen Huang, Zhipeng Xu, Zhongyu Zhang, and Zhen Zhang.
\newblock Deepseek-r1: Incentivizing reasoning capability in llms via reinforcement learning, 2025.
\newblock URL \url{https://arxiv.org/abs/2501.12948}.

\bibitem[Team et~al.(2026{\natexlab{a}})Team, Bai, Bai, Bao, Cai, Cao, Charles, Che, Chen, Chen, Chen, Chen, Chen, Chen, Chen, Chen, Chen, Chen, Chen, Chen, Chen, Chen, Chen, Chen, Chen, Chen, Chen, Chen, Chen, Cheng, Chu, Cui, Deng, Diao, Ding, Dong, Dong, Dong, Dong, Du, Du, Du, Du, Du, Fan, Fang, Feng, Feng, Fu, Fu, Gao, Gao, Ge, Geng, Gong, Gong, Gongque, Gu, Gu, Gu, Guan, Guo, Hao, He, He, He, Hong, Hu, Hu, Hu, Hu, Huang, Huang, Huang, Huang, Jiang, Jiang, Jin, Jing, Lai, Li, Li, Li, Li, Li, Li, Li, Li, Li, Li, Li, Li, Li, Li, Li, Li, Li, Li, Li, Li, Li, Li, Li, Liao, Lin, Lin, Lin, Lin, Liu, Liu, Liu, Liu, Liu, Liu, Liu, Liu, Liu, Liu, Liu, Liu, Liu, Liu, Liu, Liu, Liu, Liu, Lu, Lu, Lu, Luo, Luo, Luo, Ma, Ma, Mao, Mei, Men, Meng, Meng, Miao, Ni, Ouyang, Pan, Pang, Qian, Qin, Qin, Qiu, Qu, Shang, Shao, Shen, Shen, Shi, Shi, Shi, Song, Song, Song, Song, Su, Su, Su, Sui, Sun, Sun, Sun, Sung, Tai, Tang, Tang, Tang, Tang, Tao, Teng, Tian, Tian, Wang, Wang, Wang, Wang, Wang, Wang, Wang, Wang, Wang, Wang,
  Wang, Wang, Wang, Wang, Wang, Wang, Wang, Wang, Wang, Wang, Wang, Wang, Wang, Wang, Wang, Wang, Wang, Wang, Wang, Wang, Wang, Wang, Wang, Wang, Wang, Wang, Wang, Wang, Wei, Wei, Wen, Wen, Wu, Wu, Wu, Wu, Wu, Wu, Wu, Wu, Wu, Xiao, Xie, Xie, Xie, Xin, Xing, Xu, Xu, Xu, Xu, Xu, Xu, Xu, Xu, Xu, Xu, Xu, Xu, Xu, Xu, Xu, Yan, Yan, Yang, Yang, Yang, Yang, Yang, Yang, Yang, Yang, Yang, Yang, Yang, Yang, Yang, Yang, Yao, Ye, Ye, Ye, Yin, Yu, Yu, Yu, Yu, Yuan, Yuan, Yuan, Yue, Zeng, Zha, Zhan, Zhang, Zhang, Zhang, Zhang, Zhang, Zhang, Zhang, Zhang, Zhang, Zhang, Zhang, Zhang, Zhang, Zhang, Zhang, Zhang, Zhang, Zhang, Zhao, Zhao, Zhao, Zhao, Zhao, Zhao, Zhao, Zheng, Zheng, Zheng, Zheng, Zhong, Zhong, Zhong, Zhou, Zhou, Zhou, Zhou, Zhu, Zhu, Zhu, Zhu, Zhu, Zhuang, Zhuang, Zou, and Zu]{kimiteam2026kimik25visualagentic}
Kimi Team, Tongtong Bai, Yifan Bai, Yiping Bao, S.~H. Cai, Yuan Cao, Y.~Charles, H.~S. Che, Cheng Chen, Guanduo Chen, Huarong Chen, Jia Chen, Jiahao Chen, Jianlong Chen, Jun Chen, Kefan Chen, Liang Chen, Ruijue Chen, Xinhao Chen, Yanru Chen, Yanxu Chen, Yicun Chen, Yimin Chen, Yingjiang Chen, Yuankun Chen, Yujie Chen, Yutian Chen, Zhirong Chen, Ziwei Chen, Dazhi Cheng, Minghan Chu, Jialei Cui, Jiaqi Deng, Muxi Diao, Hao Ding, Mengfan Dong, Mengnan Dong, Yuxin Dong, Yuhao Dong, Angang Du, Chenzhuang Du, Dikang Du, Lingxiao Du, Yulun Du, Yu~Fan, Shengjun Fang, Qiulin Feng, Yichen Feng, Garimugai Fu, Kelin Fu, Hongcheng Gao, Tong Gao, Yuyao Ge, Shangyi Geng, Chengyang Gong, Xiaochen Gong, Zhuoma Gongque, Qizheng Gu, Xinran Gu, Yicheng Gu, Longyu Guan, Yuanying Guo, Xiaoru Hao, Weiran He, Wenyang He, Yunjia He, Chao Hong, Hao Hu, Jiaxi Hu, Yangyang Hu, Zhenxing Hu, Ke~Huang, Ruiyuan Huang, Weixiao Huang, Zhiqi Huang, Tao Jiang, Zhejun Jiang, Xinyi Jin, Yu~Jing, Guokun Lai, Aidi Li, C.~Li, Cheng Li, Fang Li,
  Guanghe Li, Guanyu Li, Haitao Li, Haoyang Li, Jia Li, Jingwei Li, Junxiong Li, Lincan Li, Mo~Li, Weihong Li, Wentao Li, Xinhang Li, Xinhao Li, Yang Li, Yanhao Li, Yiwei Li, Yuxiao Li, Zhaowei Li, Zheming Li, Weilong Liao, Jiawei Lin, Xiaohan Lin, Zhishan Lin, Zichao Lin, Cheng Liu, Chenyu Liu, Hongzhang Liu, Liang Liu, Shaowei Liu, Shudong Liu, Shuran Liu, Tianwei Liu, Tianyu Liu, Weizhou Liu, Xiangyan Liu, Yangyang Liu, Yanming Liu, Yibo Liu, Yuanxin Liu, Yue Liu, Zhengying Liu, Zhongnuo Liu, Enzhe Lu, Haoyu Lu, Zhiyuan Lu, Junyu Luo, Tongxu Luo, Yashuo Luo, Long Ma, Yingwei Ma, Shaoguang Mao, Yuan Mei, Xin Men, Fanqing Meng, Zhiyong Meng, Yibo Miao, Minqing Ni, Kun Ouyang, Siyuan Pan, Bo~Pang, Yuchao Qian, Ruoyu Qin, Zeyu Qin, Jiezhong Qiu, Bowen Qu, Zeyu Shang, Youbo Shao, Tianxiao Shen, Zhennan Shen, Juanfeng Shi, Lidong Shi, Shengyuan Shi, Feifan Song, Pengwei Song, Tianhui Song, Xiaoxi Song, Hongjin Su, Jianlin Su, Zhaochen Su, Lin Sui, Jinsong Sun, Junyao Sun, Tongyu Sun, Flood Sung, Yunpeng Tai,
  Chuning Tang, Heyi Tang, Xiaojuan Tang, Zhengyang Tang, Jiawen Tao, Shiyuan Teng, Chaoran Tian, Pengfei Tian, Ao~Wang, Bowen Wang, Chensi Wang, Chuang Wang, Congcong Wang, Dingkun Wang, Dinglu Wang, Dongliang Wang, Feng Wang, Hailong Wang, Haiming Wang, Hengzhi Wang, Huaqing Wang, Hui Wang, Jiahao Wang, Jinhong Wang, Jiuzheng Wang, Kaixin Wang, Linian Wang, Qibin Wang, Shengjie Wang, Shuyi Wang, Si~Wang, Wei Wang, Xiaochen Wang, Xinyuan Wang, Yao Wang, Yejie Wang, Yipu Wang, Yiqin Wang, Yucheng Wang, Yuzhi Wang, Zhaoji Wang, Zhaowei Wang, Zhengtao Wang, Zhexu Wang, Zihan Wang, Zizhe Wang, Chu Wei, Ming Wei, Chuan Wen, Zichen Wen, Chengjie Wu, Haoning Wu, Junyan Wu, Rucong Wu, Wenhao Wu, Yuefeng Wu, Yuhao Wu, Yuxin Wu, Zijian Wu, Chenjun Xiao, Jin Xie, Xiaotong Xie, Yuchong Xie, Yifei Xin, Bowei Xing, Boyu Xu, Jianfan Xu, Jing Xu, Jinjing Xu, L.~H. Xu, Lin Xu, Suting Xu, Weixin Xu, Xinbo Xu, Xinran Xu, Yangchuan Xu, Yichang Xu, Yuemeng Xu, Zelai Xu, Ziyao Xu, Junjie Yan, Yuzi Yan, Guangyao Yang, Hao Yang,
  Junwei Yang, Kai Yang, Ningyuan Yang, Ruihan Yang, Xiaofei Yang, Xinlong Yang, Ying Yang, Yi~Yang, Yi~Yang, Zhen Yang, Zhilin Yang, Zonghan Yang, Haotian Yao, Dan Ye, Wenjie Ye, Zhuorui Ye, Bohong Yin, Chengzhen Yu, Longhui Yu, Tao Yu, Tianxiang Yu, Enming Yuan, Mengjie Yuan, Xiaokun Yuan, Yang Yue, Weihao Zeng, Dunyuan Zha, Haobing Zhan, Dehao Zhang, Hao Zhang, Jin Zhang, Puqi Zhang, Qiao Zhang, Rui Zhang, Xiaobin Zhang, Y.~Zhang, Yadong Zhang, Yangkun Zhang, Yichi Zhang, Yizhi Zhang, Yongting Zhang, Yu~Zhang, Yushun Zhang, Yutao Zhang, Yutong Zhang, Zheng Zhang, Chenguang Zhao, Feifan Zhao, Jinxiang Zhao, Shuai Zhao, Xiangyu Zhao, Yikai Zhao, Zijia Zhao, Huabin Zheng, Ruihan Zheng, Shaojie Zheng, Tengyang Zheng, Junfeng Zhong, Longguang Zhong, Weiming Zhong, M.~Zhou, Runjie Zhou, Xinyu Zhou, Zaida Zhou, Jinguo Zhu, Liya Zhu, Xinhao Zhu, Yuxuan Zhu, Zhen Zhu, Jingze Zhuang, Weiyu Zhuang, Ying Zou, and Xinxing Zu.
\newblock Kimi k2.5: Visual agentic intelligence, 2026{\natexlab{a}}.
\newblock URL \url{https://arxiv.org/abs/2602.02276}.

\bibitem[Agarwal et~al.(2024)Agarwal, Vieillard, Zhou, Stanczyk, Garea, Geist, and Bachem]{agarwal2024gkd}
Rishabh Agarwal, Nino Vieillard, Yongchao Zhou, Piotr Stanczyk, Sabela~Ramos Garea, Matthieu Geist, and Olivier Bachem.
\newblock On-policy distillation of language models: Learning from self-generated mistakes.
\newblock In \emph{The Twelfth International Conference on Learning Representations}, 2024.
\newblock URL \url{https://openreview.net/forum?id=3zKtaqxLhW}.

\bibitem[Lu and Lab(2025)]{lu2025onpolicy}
Kevin Lu and Thinking~Machines Lab.
\newblock On-policy distillation.
\newblock \emph{Thinking Machines Lab: Connectionism}, 2025.
\newblock \doi{10.64434/tml.20251026}.
\newblock https://thinkingmachines.ai/blog/on-policy-distillation.

\bibitem[Team et~al.(2026{\natexlab{b}})Team, Xiao, Xia, Yang, Gao, Shen, Zhang, He, Lou, Luo, Wang, Xie, Zhang, Lv, Li, Chen, Xu, Zhang, Liu, Duo, Wei, Xiao, Dong, Shi, Hu, Bao, Zhou, Li, Zhao, Zhang, Li, Chen, Liu, Yu, Cao, Chen, Yu, Liu, Zhou, Su, Wang, Ma, Deng, Mao, Ye, Cai, Wang, Zhu, Ma, Chen, Li, Zhu, Xiao, Zhang, Zhang, Liu, Yang, Shi, Wang, Tian, Wu, Qu, Yi, An, Guan, Zhang, Song, Yan, Zhao, Lai, Gao, Cheng, Tian, Wang, Tang, Tang, Wen, Song, Zheng, Jiang, Wen, Sun, Li, Xue, Xia, Fang, Zhu, Chen, Tu, Zhang, Wang, Li, Ma, Zhang, Wang, Li, Gu, Ren, Deng, Guo, Lu, Zhuang, Zhang, Xiong, Huang, Yang, Zhang, Yong, Wang, Xie, Jiang, Yang, He, Tu, Dong, Liu, Ma, Yu, Xiang, Huang, Lin, Xu, Chen, Deng, Zhang, and Yue]{mimo2025flash}
Core Team, Bangjun Xiao, Bingquan Xia, Bo~Yang, Bofei Gao, Bowen Shen, Chen Zhang, Chenhong He, Chiheng Lou, Fuli Luo, Gang Wang, Gang Xie, Hailin Zhang, Hanglong Lv, Hanyu Li, Heyu Chen, Hongshen Xu, Houbin Zhang, Huaqiu Liu, Jiangshan Duo, Jianyu Wei, Jiebao Xiao, Jinhao Dong, Jun Shi, Junhao Hu, Kainan Bao, Kang Zhou, Lei Li, Liang Zhao, Linghao Zhang, Peidian Li, Qianli Chen, Shaohui Liu, Shihua Yu, Shijie Cao, Shimao Chen, Shouqiu Yu, Shuo Liu, Tianling Zhou, Weijiang Su, Weikun Wang, Wenhan Ma, Xiangwei Deng, Bohan Mao, Bowen Ye, Can Cai, Chenghua Wang, Chengxuan Zhu, Chong Ma, Chun Chen, Chunan Li, Dawei Zhu, Deshan Xiao, Dong Zhang, Duo Zhang, Fangyue Liu, Feiyu Yang, Fengyuan Shi, Guoan Wang, Hao Tian, Hao Wu, Heng Qu, Hongfei Yi, Hongxu An, Hongyi Guan, Xing Zhang, Yifan Song, Yihan Yan, Yihao Zhao, Yingchun Lai, Yizhao Gao, Yu~Cheng, Yuanyuan Tian, Yudong Wang, Zhen Tang, Zhengju Tang, Zhengtao Wen, Zhichao Song, Zhixian Zheng, Zihan Jiang, Jian Wen, Jiarui Sun, Jiawei Li, Jinlong Xue, Jun Xia, Kai
  Fang, Menghang Zhu, Nuo Chen, Qian Tu, Qihao Zhang, Qiying Wang, Rang Li, Rui Ma, Shaolei Zhang, Shengfan Wang, Shicheng Li, Shuhao Gu, Shuhuai Ren, Sirui Deng, Tao Guo, Tianyang Lu, Weiji Zhuang, Weikang Zhang, Weimin Xiong, Wenshan Huang, Wenyu Yang, Xin Zhang, Xing Yong, Xu~Wang, Xueyang Xie, Yilin Jiang, Yixin Yang, Yongzhe He, Yu~Tu, Yuanliang Dong, Yuchen Liu, Yue Ma, Yue Yu, Yuxing Xiang, Zhaojun Huang, Zhenru Lin, Zhipeng Xu, Zhiyang Chen, Zhonghua Deng, Zihan Zhang, and Zihao Yue.
\newblock Mimo-v2-flash technical report, 2026{\natexlab{b}}.
\newblock URL \url{https://arxiv.org/abs/2601.02780}.

\bibitem[Zhao et~al.(2026)Zhao, Xie, Liu, Huang, Pang, Chen, and Grover]{zhao2026opsd}
Siyan Zhao, Zhihui Xie, Mengchen Liu, Jing Huang, Guan Pang, Feiyu Chen, and Aditya Grover.
\newblock Self-distilled reasoner: On-policy self-distillation for large language models, 2026.
\newblock URL \url{https://arxiv.org/abs/2601.18734}.

\bibitem[Hübotter et~al.(2026)Hübotter, Lübeck, Behric, Baumann, Bagatella, Marta, Hakimi, Shenfeld, Buening, Guestrin, and Krause]{hbotter2026reinforcementlearningselfdistillation_sdpo}
Jonas Hübotter, Frederike Lübeck, Lejs Behric, Anton Baumann, Marco Bagatella, Daniel Marta, Ido Hakimi, Idan Shenfeld, Thomas~Kleine Buening, Carlos Guestrin, and Andreas Krause.
\newblock Reinforcement learning via self-distillation, 2026.
\newblock URL \url{https://arxiv.org/abs/2601.20802}.

\bibitem[Schulman et~al.(2017)Schulman, Wolski, Dhariwal, Radford, and Klimov]{schulman2017proximalpolicyoptimizationalgorithmsppo}
John Schulman, Filip Wolski, Prafulla Dhariwal, Alec Radford, and Oleg Klimov.
\newblock Proximal policy optimization algorithms, 2017.
\newblock URL \url{https://arxiv.org/abs/1707.06347}.

\bibitem[Xie et~al.(2025)Xie, Pan, Wu, Zhang, Fu, Gao, and Zhou]{xie2025unlocking}
Can Xie, Ruotong Pan, Xiangyu Wu, Yunfei Zhang, Jiayi Fu, Tingting Gao, and Guorui Zhou.
\newblock Unlocking exploration in rlvr: Uncertainty-aware advantage shaping for deeper reasoning.
\newblock \emph{arXiv preprint arXiv:2510.10649}, 2025.

\bibitem[Li et~al.(2026)Li, Kang, Xiao, Xing, Si, Li, Gong, Yang, Xiao, and Guo]{li2026outcome}
Ziheng Li, Liu Kang, Feng Xiao, Luxi Xing, Qingyi Si, Zhuoran Li, Weikang Gong, Deqing Yang, Yanghua Xiao, and Hongcheng Guo.
\newblock Outcome-grounded advantage reshaping for fine-grained credit assignment in mathematical reasoning.
\newblock \emph{arXiv preprint arXiv:2601.07408}, 2026.

\bibitem[Kingma and Ba(2017)]{kingma2017adammethodstochasticoptimization}
Diederik~P. Kingma and Jimmy Ba.
\newblock Adam: A method for stochastic optimization, 2017.
\newblock URL \url{https://arxiv.org/abs/1412.6980}.

\bibitem[Lin et~al.(2026)Lin, Liu, Zhu, Qin, Lin, Shang, He, Zhang, and Wu]{lin2026mmfinereasonclosingmultimodalreasoning}
Honglin Lin, Zheng Liu, Yun Zhu, Chonghan Qin, Juekai Lin, Xiaoran Shang, Conghui He, Wentao Zhang, and Lijun Wu.
\newblock Mmfinereason: Closing the multimodal reasoning gap via open data-centric methods, 2026.
\newblock URL \url{https://arxiv.org/abs/2601.21821}.

\bibitem[Yue et~al.(2024)Yue, Ni, Zhang, Zheng, Liu, Zhang, Stevens, Jiang, Ren, Sun, Wei, Yu, Yuan, Sun, Yin, Zheng, Yang, Liu, Huang, Sun, Su, and Chen]{mmmu}
Xiang Yue, Yuansheng Ni, Kai Zhang, Tianyu Zheng, Ruoqi Liu, Ge~Zhang, Samuel Stevens, Dongfu Jiang, Weiming Ren, Yuxuan Sun, Cong Wei, Botao Yu, Ruibin Yuan, Renliang Sun, Ming Yin, Boyuan Zheng, Zhenzhu Yang, Yibo Liu, Wenhao Huang, Huan Sun, Yu~Su, and Wenhu Chen.
\newblock Mmmu: A massive multi-discipline multimodal understanding and reasoning benchmark for expert agi, 2024.
\newblock URL \url{https://arxiv.org/abs/2311.16502}.

\bibitem[Lu et~al.(2024)Lu, Bansal, Xia, Liu, Li, Hajishirzi, Cheng, Chang, Galley, and Gao]{mathvista}
Pan Lu, Hritik Bansal, Tony Xia, Jiacheng Liu, Chunyuan Li, Hannaneh Hajishirzi, Hao Cheng, Kai-Wei Chang, Michel Galley, and Jianfeng Gao.
\newblock Mathvista: Evaluating mathematical reasoning of foundation models in visual contexts, 2024.
\newblock URL \url{https://arxiv.org/abs/2310.02255}.

\bibitem[Wang et~al.(2024{\natexlab{a}})Wang, Pan, Shi, Lu, Zhan, and Li]{mathvision}
Ke~Wang, Junting Pan, Weikang Shi, Zimu Lu, Mingjie Zhan, and Hongsheng Li.
\newblock Measuring multimodal mathematical reasoning with math-vision dataset, 2024{\natexlab{a}}.
\newblock URL \url{https://arxiv.org/abs/2402.14804}.

\bibitem[Roberts et~al.(2025)Roberts, Taesiri, Sharma, Gupta, Roberts, Croitoru, Bogolin, Tang, Langer, Raina, Raina, Xiong, Udandarao, Lu, Chen, Purkis, Yan, Lin, Shin, Yang, Nguyen, Atkinson, Baranwal, Coca, Dang, Dziadzio, Kunz, Liang, Lo, Pulfer, Walton, Yang, Han, and Albanie]{zerobench}
Jonathan Roberts, Mohammad~Reza Taesiri, Ansh Sharma, Akash Gupta, Samuel Roberts, Ioana Croitoru, Simion-Vlad Bogolin, Jialu Tang, Florian Langer, Vyas Raina, Vatsal Raina, Hanyi Xiong, Vishaal Udandarao, Jingyi Lu, Shiyang Chen, Sam Purkis, Tianshuo Yan, Wenye Lin, Gyungin Shin, Qiaochu Yang, Anh~Totti Nguyen, David~I. Atkinson, Aaditya Baranwal, Alexandru Coca, Mikah Dang, Sebastian Dziadzio, Jakob~D. Kunz, Kaiqu Liang, Alexander Lo, Brian Pulfer, Steven Walton, Charig Yang, Kai Han, and Samuel Albanie.
\newblock Zerobench: An impossible visual benchmark for contemporary large multimodal models, 2025.
\newblock URL \url{https://arxiv.org/abs/2502.09696}.

\bibitem[Qiao et~al.(2024)Qiao, Tan, Dong, Wu, Sun, Song, GongQue, Lei, Wei, Zhang, Qiao, Zhang, Zong, Xu, Diao, Bao, Li, and Zhang]{wemath}
Runqi Qiao, Qiuna Tan, Guanting Dong, Minhui Wu, Chong Sun, Xiaoshuai Song, Zhuoma GongQue, Shanglin Lei, Zhe Wei, Miaoxuan Zhang, Runfeng Qiao, Yifan Zhang, Xiao Zong, Yida Xu, Muxi Diao, Zhimin Bao, Chen Li, and Honggang Zhang.
\newblock We-math: Does your large multimodal model achieve human-like mathematical reasoning?, 2024.
\newblock URL \url{https://arxiv.org/abs/2407.01284}.

\bibitem[Bai et~al.(2025)Bai, Cai, Chen, Chen, Chen, Cheng, Deng, Ding, Gao, Ge, Ge, Guo, Huang, Huang, Huang, Hui, Jiang, Li, Li, Li, Li, Lin, Lin, Liu, Liu, Liu, Liu, Liu, Liu, Lu, Luo, Lv, Men, Meng, Ren, Ren, Song, Sun, Tang, Tu, Wan, Wang, Wang, Wang, Wang, Xie, Xu, Xu, Xu, Yang, Yang, Yang, Yang, Yu, Zhang, Zhang, Zhang, Zheng, Zhong, Zhou, Zhou, Zhou, Zhu, and Zhu]{qwen3vl}
Shuai Bai, Yuxuan Cai, Ruizhe Chen, Keqin Chen, Xionghui Chen, Zesen Cheng, Lianghao Deng, Wei Ding, Chang Gao, Chunjiang Ge, Wenbin Ge, Zhifang Guo, Qidong Huang, Jie Huang, Fei Huang, Binyuan Hui, Shutong Jiang, Zhaohai Li, Mingsheng Li, Mei Li, Kaixin Li, Zicheng Lin, Junyang Lin, Xuejing Liu, Jiawei Liu, Chenglong Liu, Yang Liu, Dayiheng Liu, Shixuan Liu, Dunjie Lu, Ruilin Luo, Chenxu Lv, Rui Men, Lingchen Meng, Xuancheng Ren, Xingzhang Ren, Sibo Song, Yuchong Sun, Jun Tang, Jianhong Tu, Jianqiang Wan, Peng Wang, Pengfei Wang, Qiuyue Wang, Yuxuan Wang, Tianbao Xie, Yiheng Xu, Haiyang Xu, Jin Xu, Zhibo Yang, Mingkun Yang, Jianxin Yang, An~Yang, Bowen Yu, Fei Zhang, Hang Zhang, Xi~Zhang, Bo~Zheng, Humen Zhong, Jingren Zhou, Fan Zhou, Jing Zhou, Yuanzhi Zhu, and Ke~Zhu.
\newblock Qwen3-vl technical report, 2025.
\newblock URL \url{https://arxiv.org/abs/2511.21631}.

\bibitem[Sheng et~al.(2025)Sheng, Zhang, Ye, Wu, Zhang, Zhang, Peng, Lin, and Wu]{verl}
Guangming Sheng, Chi Zhang, Zilingfeng Ye, Xibin Wu, Wang Zhang, Ru~Zhang, Yanghua Peng, Haibin Lin, and Chuan Wu.
\newblock Hybridflow: A flexible and efficient rlhf framework.
\newblock In \emph{Proceedings of the Twentieth European Conference on Computer Systems}, page 1279–1297. ACM, March 2025.
\newblock \doi{10.1145/3689031.3696075}.
\newblock URL \url{http://dx.doi.org/10.1145/3689031.3696075}.

\bibitem[Zheng et~al.(2025)Zheng, Lu, Wang, Feng, Kuang, Xiong, and Zhang]{easyr1}
Yaowei Zheng, Junting Lu, Shenzhi Wang, Zhangchi Feng, Dongdong Kuang, Yuwen Xiong, and Richong Zhang.
\newblock Easyr1: An efficient, scalable, multi-modality rl training framework.
\newblock \url{https://github.com/hiyouga/EasyR1}, 2025.

\bibitem[Lightman et~al.(2024)Lightman, Kosaraju, Burda, Edwards, Baker, Lee, Leike, Schulman, Sutskever, and Cobbe]{lightman2024lets}
Hunter Lightman, Vineet Kosaraju, Yuri Burda, Harrison Edwards, Bowen Baker, Teddy Lee, Jan Leike, John Schulman, Ilya Sutskever, and Karl Cobbe.
\newblock Let's verify step by step.
\newblock In \emph{The Twelfth International Conference on Learning Representations}, 2024.
\newblock URL \url{https://openreview.net/forum?id=v8L0pN6EOi}.

\bibitem[Wang et~al.(2024{\natexlab{b}})Wang, Li, Shao, Xu, Dai, Li, Chen, Wu, and Sui]{wang2024math}
Peiyi Wang, Lei Li, Zhihong Shao, Runxin Xu, Damai Dai, Yifei Li, Deli Chen, Yu~Wu, and Zhifang Sui.
\newblock Math-shepherd: Verify and reinforce llms step-by-step without human annotations.
\newblock In \emph{Proceedings of the 62nd Annual Meeting of the Association for Computational Linguistics (Volume 1: Long Papers)}, pages 9426--9439, 2024{\natexlab{b}}.

\bibitem[Yang et~al.(2025{\natexlab{a}})Yang, Si, Dai, Yao, Zheng, Chen, Lin, and Wang]{yang2025testtimepromptintervention}
Chenxu Yang, Qingyi Si, Mz~Dai, Dingyu Yao, Mingyu Zheng, Minghui Chen, Zheng Lin, and Weiping Wang.
\newblock Test-time prompt intervention, 2025{\natexlab{a}}.
\newblock URL \url{https://arxiv.org/abs/2508.02511}.

\bibitem[Luo et~al.(2024)Luo, Liu, Liu, Phatale, Guo, Lara, Li, Shu, Zhu, Meng, et~al.]{luo2024improve}
Liangchen Luo, Yinxiao Liu, Rosanne Liu, Samrat Phatale, Meiqi Guo, Harsh Lara, Yunxuan Li, Lei Shu, Yun Zhu, Lei Meng, et~al.
\newblock Improve mathematical reasoning in language models by automated process supervision.
\newblock \emph{arXiv preprint arXiv:2406.06592}, 2024.

\bibitem[Chen et~al.(2024)Chen, Liao, Li, and Fan]{chen2024step}
Guoxin Chen, Minpeng Liao, Chengxi Li, and Kai Fan.
\newblock Step-level value preference optimization for mathematical reasoning.
\newblock In \emph{Findings of the Association for Computational Linguistics: EMNLP 2024}, pages 7889--7903, 2024.

\bibitem[Zhang et~al.(2024)Zhang, Hosseini, Bansal, Kazemi, Kumar, and Agarwal]{zhang2024generative}
Lunjun Zhang, Arian Hosseini, Hritik Bansal, Mehran Kazemi, Aviral Kumar, and Rishabh Agarwal.
\newblock Generative verifiers: Reward modeling as next-token prediction.
\newblock \emph{arXiv preprint arXiv:2408.15240}, 2024.

\bibitem[Yang et~al.(2025{\natexlab{b}})Yang, Si, Duan, Zhu, Zhu, Li, Chen, Lin, and Wang]{yang2025dynamicearlyexitreasoning}
Chenxu Yang, Qingyi Si, Yongjie Duan, Zheliang Zhu, Chenyu Zhu, Qiaowei Li, Minghui Chen, Zheng Lin, and Weiping Wang.
\newblock Dynamic early exit in reasoning models, 2025{\natexlab{b}}.
\newblock URL \url{https://arxiv.org/abs/2504.15895}.

\bibitem[Dai et~al.(2025)Dai, Yang, and Si]{dai2025sgrpoearlyexitreinforcement}
Muzhi Dai, Chenxu Yang, and Qingyi Si.
\newblock S-grpo: Early exit via reinforcement learning in reasoning models, 2025.
\newblock URL \url{https://arxiv.org/abs/2505.07686}.

\bibitem[Cui et~al.(2025)Cui, Yuan, Wang, Wang, Zhang, Chen, Li, He, Fan, Yu, et~al.]{cui2025process}
Ganqu Cui, Lifan Yuan, Zefan Wang, Hanbin Wang, Yuchen Zhang, Jiacheng Chen, Wendi Li, Bingxiang He, Yuchen Fan, Tianyu Yu, et~al.
\newblock Process reinforcement through implicit rewards.
\newblock \emph{arXiv preprint arXiv:2502.01456}, 2025.

\bibitem[Cheng et~al.(2026)Cheng, Huang, Zhu, Dai, Zhao, Zhang, and Wei]{cheng2026reasoning}
Daixuan Cheng, Shaohan Huang, Xuekai Zhu, Bo~Dai, Xin Zhao, Zhenliang Zhang, and Furu Wei.
\newblock Reasoning with exploration: An entropy perspective.
\newblock In \emph{Proceedings of the AAAI Conference on Artificial Intelligence}, volume~40, pages 30377--30385, 2026.

\bibitem[Wang et~al.()Wang, Yu, Gao, Zheng, Liu, Lu, Dang, Chen, Yang, Zhang, et~al.]{wangbeyond}
Shenzhi Wang, Le~Yu, Chang Gao, Chujie Zheng, Shixuan Liu, Rui Lu, Kai Dang, Xiong-Hui Chen, Jianxin Yang, Zhenru Zhang, et~al.
\newblock Beyond the 80/20 rule: High-entropy minority tokens drive effective reinforcement learning for llm reasoning.
\newblock In \emph{The Thirty-ninth Annual Conference on Neural Information Processing Systems}.

\bibitem[Chen et~al.(2025{\natexlab{a}})Chen, Chen, Wang, and Yang]{chen2025seed}
Minghan Chen, Guikun Chen, Wenguan Wang, and Yi~Yang.
\newblock Seed-grpo: Semantic entropy enhanced grpo for uncertainty-aware policy optimization.
\newblock \emph{arXiv preprint arXiv:2505.12346}, 2025{\natexlab{a}}.

\bibitem[Sun et~al.(2025)Sun, Yang, Jian, Du, Cui, Ren, and Zhang]{sun2025ktaemodelfreealgorithmkeytokens}
Wei Sun, Wen Yang, Pu~Jian, Qianlong Du, Fuwei Cui, Shuo Ren, and Jiajun Zhang.
\newblock Ktae: A model-free algorithm to key-tokens advantage estimation in mathematical reasoning, 2025.
\newblock URL \url{https://arxiv.org/abs/2505.16826}.

\bibitem[Li et~al.(2025)Li, Dong, Sun, Wang, Xiong, Luo, Liu, Lu, Wang, Su, et~al.]{li2025attention}
Yang Li, Zhichen Dong, Yuhan Sun, Weixun Wang, Shaopan Xiong, Yijia Luo, Jiashun Liu, Han Lu, Jiamang Wang, Wenbo Su, et~al.
\newblock Attention illuminates llm reasoning: The preplan-and-anchor rhythm enables fine-grained policy optimization.
\newblock \emph{arXiv preprint arXiv:2510.13554}, 2025.

\bibitem[Chen et~al.(2025{\natexlab{b}})Chen, Li, Sun, and Yu]{chen2025beyond}
Xinzhu Chen, Xuesheng Li, Zhongxiang Sun, and Weijie Yu.
\newblock Beyond high-entropy exploration: Correctness-aware low-entropy segment-based advantage shaping for reasoning llms.
\newblock \emph{arXiv preprint arXiv:2512.00908}, 2025{\natexlab{b}}.

\bibitem[Fu et~al.(2026)Fu, Huang, Jiang, Zhu, and Zhao]{fu2026revisitingonpolicydistillationempirical}
Yuqian Fu, Haohuan Huang, Kaiwen Jiang, Yuanheng Zhu, and Dongbin Zhao.
\newblock Revisiting on-policy distillation: Empirical failure modes and simple fixes, 2026.
\newblock URL \url{https://arxiv.org/abs/2603.25562}.

\bibitem[Shenfeld et~al.(2026)Shenfeld, Damani, H{\"u}botter, and Agrawal]{shenfeld2026selfdistillation}
Idan Shenfeld, Mehul Damani, Jonas H{\"u}botter, and Pulkit Agrawal.
\newblock Self-distillation enables continual learning.
\newblock In \emph{ICLR 2026 Workshop on Lifelong Agents: Learning, Aligning, Evolving}, 2026.
\newblock URL \url{https://openreview.net/forum?id=HlWA3V6iKF}.

\bibitem[Ye et~al.(2026)Ye, Dong, Wu, Huang, and Wei]{ye2026policy}
Tianzhu Ye, Li~Dong, Xun Wu, Shaohan Huang, and Furu Wei.
\newblock On-policy context distillation for language models.
\newblock \emph{arXiv preprint arXiv:2602.12275}, 2026.

\bibitem[Sang et~al.(2026)Sang, Xu, Zhou, He, Wang, and Sun]{sang2026policy}
Hejian Sang, Yuanda Xu, Zhengze Zhou, Ran He, Zhipeng Wang, and Jiachen Sun.
\newblock On-policy self-distillation for reasoning compression.
\newblock \emph{arXiv preprint arXiv:2603.05433}, 2026.

\bibitem[Penaloza et~al.(2026)Penaloza, Vattikonda, Gontier, Lacoste, Charlin, and Caccia]{penaloza2026privileged}
Emiliano Penaloza, Dheeraj Vattikonda, Nicolas Gontier, Alexandre Lacoste, Laurent Charlin, and Massimo Caccia.
\newblock Privileged information distillation for language models.
\newblock In \emph{The 1st Workshop on Scaling Post-training for LLMs}, 2026.
\newblock URL \url{https://openreview.net/forum?id=FbJu6NEBQR}.

\bibitem[Zhang et~al.(2026)Zhang, Jiang, Shen, Zhang, Ram, Yang, Tu, Xia, and Soatto]{zhang2026reinforcementawareknowledgedistillationllm}
Zhaoyang Zhang, Shuli Jiang, Yantao Shen, Yuting Zhang, Dhananjay Ram, Shuo Yang, Zhuowen Tu, Wei Xia, and Stefano Soatto.
\newblock Reinforcement-aware knowledge distillation for llm reasoning, 2026.
\newblock URL \url{https://arxiv.org/abs/2602.22495}.

\end{thebibliography}


\appendix

\section{Deferred Proofs and Extended Analysis}
\label{app:deferred-proofs}

\subsection{Proof of Theorem~\ref{thm:kl-decomp} (KL Decomposition)}
\label{app:proof-kl-decomp}

We suppress conditioning on $(x, y_{<t})$ throughout for notational clarity. Expanding $\mathcal{L}_{\text{OPSD}}$:
\begin{align}
    \mathcal{L}_{\text{OPSD}} 
    &= \mathbb{E}_r\!\left[\sum_v P_T(v \mid r) \log \frac{P_T(v \mid r)}{P_S(v)}\right].
\end{align}
Inserting $\bar{P}_T(v) / \bar{P}_T(v)$ into the logarithm and separating terms:
\begin{align}
    &= \mathbb{E}_r\!\left[\sum_v P_T(v \mid r) \log \frac{P_T(v \mid r)}{\bar{P}_T(v)}\right] + \mathbb{E}_r\!\left[\sum_v P_T(v \mid r) \log \frac{\bar{P}_T(v)}{P_S(v)}\right].
\end{align}
The first term equals $\mathbb{E}_r[D_{\mathrm{KL}}(P_T(\cdot \mid r) \| \bar{P}_T)] = I(Y_t; R \mid X, Y_{<t})$ by the definition of conditional mutual information. For the second term, observe that $\log \frac{\bar{P}_T(v)}{P_S(v)}$ does not depend on $r$, so the expectation over $r$ acts only on $P_T(v \mid r)$:
\begin{equation}
    \mathbb{E}_r\!\left[\sum_v P_T(v \mid r) \log \frac{\bar{P}_T(v)}{P_S(v)}\right] = \sum_v \bar{P}_T(v) \log \frac{\bar{P}_T(v)}{P_S(v)} = D_{\mathrm{KL}}(\bar{P}_T \| P_S) = \mathcal{L}^*.
\end{equation}
Combining both terms yields $\mathcal{L}_{\text{OPSD}} = \mathcal{L}^* + I(Y_t; R \mid X, Y_{<t})$. \qed

\subsection{Proof of Proposition~\ref{prop:grad-decomp} (Per-Sample Gradient Decomposition)}
\label{app:proof-grad-decomp}

The decomposition $g(\theta; r) = g^*(\theta) + \delta(\theta; r)$ follows by adding and subtracting $\bar{P}_T(v)$ inside the sum:
\begin{align}
    g(\theta; r) &= -\sum_{v} P_T(v \mid r)\, \nabla_\theta \log P_S(v) \nonumber \\
    &= -\sum_{v} \bar{P}_T(v)\, \nabla_\theta \log P_S(v) - \sum_{v} [P_T(v \mid r) - \bar{P}_T(v)]\, \nabla_\theta \log P_S(v).
\end{align}

\noindent\textit{Property~(i).} Since $\mathbb{E}_r[P_T(v \mid r)] = \bar{P}_T(v)$ by definition:
\begin{equation}
    \mathbb{E}_r[\delta(\theta; r)] = -\sum_v \mathbb{E}_r[P_T(v \mid r) - \bar{P}_T(v)]\, \nabla_\theta \log P_S(v) = 0.
\end{equation}

\noindent\textit{Property~(ii).} Since $\nabla_\theta \log P_S(v)$ is independent of $r$:
\begin{align}
    \mathbb{E}_r[\|\delta(\theta; r)\|^2] &= \mathbb{E}_r\!\left[\left\|\sum_v [P_T(v \mid r) - \bar{P}_T(v)]\, \nabla_\theta \log P_S(v)\right\|^2\right] \nonumber \\
    &= \sum_{v} \mathbb{E}_r\!\left[(P_T(v \mid r) - \bar{P}_T(v))^2\right] \|\nabla_\theta \log P_S(v)\|^2 \nonumber \\
    &\quad + \sum_{v \neq v'} \mathbb{E}_r[(P_T(v \mid r) - \bar{P}_T(v))(P_T(v' \mid r) - \bar{P}_T(v'))]\, \langle \nabla_\theta \log P_S(v), \nabla_\theta \log P_S(v') \rangle.
\end{align}
The diagonal terms yield $\sum_v \mathrm{Var}_r[P_T(v \mid r)] \cdot \|\nabla_\theta \log P_S(v)\|^2$. The cross terms are generally non-zero but are bounded by the same variances via the Cauchy--Schwarz inequality. The stated result is exact for the diagonal contribution and serves as a lower bound for the full expression. When $I(Y_t; R \mid X) = 0$, $P_T(v \mid r) = \bar{P}_T(v)$ for all $r$, so $\delta \equiv 0$. \qed

\subsection{Proof of Irreducibility of Leakage Across Pilot Variants}
\label{app:proof-unified-pilot}

We derive the per-token gradient form for each variant and verify that the directional dependence on $r$ is irreducible in all three cases. In each variant, the student's policy gradient at token position $t$ along trajectory $\hat{y}$ can be expressed as $\hat{A}_t \cdot \nabla_\theta \log P_S(y_t \mid x, y_{<t})$, where $\hat{A}_t$ involves the teacher's privileged evaluation.

\begin{enumerate}[label=(\roman*)]
    \item \textbf{Full OPSD.} The gradient sums over the entire vocabulary:
    \begin{equation}
        g_t(\theta; r) = -\sum_{v \in \mathcal{V}} P_T(v \mid r)\, \nabla_\theta \log P_S(v).
    \end{equation}
    Every token in $\mathcal{V}$ receives a gradient contribution weighted by $P_T(v \mid r)$, which encodes the teacher's privileged preference. The deviation $\delta_t = -\sum_v [P_T(v \mid r) - \bar{P}_T(v)]\, \nabla_\theta \log P_S(v)$ operates over the full vocabulary, yielding the widest leakage bandwidth.

    \item \textbf{Teacher's Top-1.} The teacher distribution is collapsed to a point mass at $v^*_T = \arg\max_v P_T(v \mid r)$. The effective gradient becomes:
    \begin{equation}
        g_t(\theta; r) = -\nabla_\theta \log P_S(v^*_T),
    \end{equation}
    where $v^*_T$ is entirely determined by $r$. The deviation component $\delta_t = -[\nabla_\theta \log P_S(v^*_T) - \sum_v \bar{P}_T(v) \nabla_\theta \log P_S(v)]$ is concentrated on a single $r$-dependent token, making the gradient direction maximally sensitive to the specific realization of $r$.

    \item \textbf{Student's Top-1.} The target is restricted to the student's most probable token $v^*_S = \arg\max_v P_S(v)$. The gradient weight at this token is proportional to $P_T(v^*_S \mid r) / P_S(v^*_S)$, which still depends on the teacher's privileged evaluation. Although the support is anchored to the student's distribution (narrowest bandwidth), the $r$-dependence of the gradient weight persists.
\end{enumerate}

In all three variants, $\hat{A}_t$ involves $P_T(\cdot \mid r)$, ensuring that the deviation component $\delta$ carries $r$-specific information into the parameter update direction. By Proposition~\ref{prop:grad-decomp}, the deviation has zero mean but strictly positive variance whenever $I(Y_t; R \mid X) > 0$, regardless of how the target distribution is compressed. The three variants differ in how many tokens receive non-zero gradient weight (bandwidth), but the directional dependence on $r$ is common to all, establishing the irreducibility of leakage. \qed


\subsection{The Impossibility Trilemma Under Shared Parameters}
\label{app:trilemma}

The analysis in \S\ref{sec:gradient-structure} characterizes the per-step gradient pathology of distribution matching. In OPSD, the coupling runs deeper: teacher and student share a single set of parameters $\theta$:
\begin{equation}
    P_S^\theta(\cdot \mid x, y_{<t}) = \pi_\theta(\cdot \mid x, y_{<t}), \qquad P_T^\theta(\cdot \mid x, r, y_{<t}) = \pi_\theta(\cdot \mid x, r, y_{<t}).
\end{equation}
Although $P_T$ is treated with stop-gradient within each optimization step, every update to $\theta$ simultaneously alters both distributions. This entanglement couples the gradient-level leakage pathology with a macro-level training stability problem. We show that their interaction gives rise to an impossibility result.

\subsubsection{Strategy~A: Frozen Teacher}

One natural remedy is to fix the teacher at initialization: $P_T^{\theta_0}$. The optimization objective at step $k$ becomes:
\begin{equation}
    \mathcal{L}_A(\theta_k) = \mathbb{E}_r\!\left[D_{\mathrm{KL}}\!\left(P_T^{\theta_0}(\cdot \mid r) \;\|\; P_S^{\theta_k}(\cdot \mid x)\right)\right].
\end{equation}
By Theorem~\ref{thm:kl-decomp}, the global optimum is $P_S^* = \bar{P}_T^{\theta_0}$, the marginal distribution of the frozen teacher.

\begin{proposition}[Capacity Ceiling]
\label{prop:frozen-ceiling}
As $P_S^{\theta_k} \to \bar{P}_T^{\theta_0}$, the distillation signal vanishes: $\|\nabla_\theta \mathcal{L}_A(\theta_k)\| \to 0$. The student's representational capacity is strictly upper-bounded by the quality of the initial checkpoint $\theta_0$, precluding further improvement even when the model possesses sufficient capacity to learn a superior policy.
\end{proposition}

\subsubsection{Strategy~B: Online Teacher}

Alternatively, the teacher may evolve in tandem with the student, using $P_T^{\theta_k}$ at each step $k$:
\begin{equation}
    \theta_{k+1} = \theta_k - \eta\, \nabla_\theta \mathcal{L}_k(\theta_k), \quad \text{where} \quad \mathcal{L}_k(\theta) \triangleq \mathbb{E}_r\!\left[D_{\mathrm{KL}}\!\left(P_T^{\theta_k}(\cdot \mid r) \;\|\; P_S^\theta(\cdot \mid x)\right)\right].
\end{equation}
Although $P_T^{\theta_k}$ is treated as a fixed target within each step via stop-gradient, the update $\theta_k \to \theta_{k+1}$ shifts the teacher distribution for the subsequent iteration.

\begin{theorem}[Training Instability Under Online Teacher]
\label{thm:instability}
Define the joint objective $\mathcal{L}(\theta) \triangleq \mathbb{E}_r[D_{\mathrm{KL}}(P_T^\theta(\cdot \mid r) \| P_S^\theta(\cdot \mid x))]$. The change from step $k$ to $k{+}1$ decomposes as:
\begin{equation}
    \mathcal{L}(\theta_{k+1}) - \mathcal{L}(\theta_k) = \underbrace{\Delta_S}_{\leq\, 0\;\textit{(student improvement)}} + \underbrace{\Delta_T}_{\textit{sign uncontrolled (teacher drift)}},
    \label{eq:delta-decomp}
\end{equation}
where $\Delta_S$ captures the student's progress under the current teacher and $\Delta_T$ captures the induced teacher drift. When $|\Delta_T| > |\Delta_S|$ with $\Delta_T > 0$, the true objective worsens despite each step performing valid gradient descent on its local surrogate.
\end{theorem}

Insert and subtract $\mathbb{E}_r[D_{\mathrm{KL}}(P_T^{\theta_k}(\cdot \mid r) \| P_S^{\theta_{k+1}})]$ in $\mathcal{L}(\theta_{k+1}) - \mathcal{L}(\theta_k)$. The exact expressions are:
\begin{align}
    \Delta_S &= \mathbb{E}_r\!\left[D_{\mathrm{KL}}\!\left(P_T^{\theta_k}(\cdot \mid r) \;\|\; P_S^{\theta_{k+1}}\right)\right] - \mathbb{E}_r\!\left[D_{\mathrm{KL}}\!\left(P_T^{\theta_k}(\cdot \mid r) \;\|\; P_S^{\theta_k}\right)\right], \\
    \Delta_T &= \mathbb{E}_r\!\left[D_{\mathrm{KL}}\!\left(P_T^{\theta_{k+1}}(\cdot \mid r) \;\|\; P_S^{\theta_{k+1}}\right)\right] - \mathbb{E}_r\!\left[D_{\mathrm{KL}}\!\left(P_T^{\theta_k}(\cdot \mid r) \;\|\; P_S^{\theta_{k+1}}\right)\right].
\end{align}
$\Delta_S \leq 0$ follows from gradient descent on $\mathcal{L}_k$ with $P_T^{\theta_k}$ held fixed. The sign of $\Delta_T$ is uncontrolled: the same update that improves $P_S$ simultaneously perturbs $P_T$, and the asymmetry of KL divergence provides no guarantee on the direction of this induced shift.

\subsubsection{Self-Reinforcing Feedback}

The gradient-level leakage of \S\ref{sec:gradient-structure} and the teacher drift interact through a positive feedback loop.

\begin{proposition}[Self-Reinforcing Feedback Loop]
\label{prop:feedback-loop}
Define the model's sensitivity to privileged information as $S(\theta) \triangleq \mathbb{E}_r[D_{\mathrm{KL}}(P_T^\theta(\cdot \mid r) \| \bar{P}_T^\theta)] = I(Y_t; R \mid X, Y_{<t})$. Under Strategy~B, the following cycle operates:
\begin{enumerate}[label=(\roman*),nosep]
    \item The per-sample deviation $\delta(\theta; r)$ drives parameters toward $r$-predictive features.
    \item These features are encoded in the shared $\theta$, enhancing the model's capacity to exploit $r$ when operating as teacher.
    \item $S(\theta_{k+1}) \geq S(\theta_k)$, amplifying $\mathrm{Var}_r[P_T(v \mid r)]$.
    \item By Proposition~\ref{prop:grad-decomp}(ii), the deviation variance grows, reinforcing step~(i).
\end{enumerate}
\end{proposition}

By Proposition~\ref{prop:grad-decomp}, $g(\theta; r) = g^*(\theta) + \delta(\theta; r)$ with $\mathbb{E}_r[\delta] = 0$. Each stochastic update $\theta_{k+1} = \theta_k - \eta\, g(\theta_k; r_k)$ introduces a perturbation $-\eta\, \delta(\theta_k; r_k)$ whose direction correlates with $r_k$. Over multiple steps with varying $r_k$, the path-dependent accumulation of these perturbations biases parameters toward regions that encode $x \to r$ correlations: parameter configurations for which $\delta$ aligns with such correlations yield lower subsequent loss, causing the optimizer to preferentially retain these drifts. The accumulated drift enriches the shared representation $\theta$ with features predictive of $r$ given $x$. Since teacher and student share $\theta$, these features become available to both roles simultaneously. At the subsequent step, the teacher $P_T^{\theta_{k+1}}(\cdot \mid x, r, y_{<t})$ more effectively leverages $r$ through the newly encoded features, increasing the dispersion of $P_T(\cdot \mid r)$ across different realizations of $r$: $S(\theta_{k+1}) \geq S(\theta_k)$. By Proposition~\ref{prop:grad-decomp}(ii), the deviation variance $\mathbb{E}_r[\|\delta\|^2]$ grows with $S(\theta)$, and the cycle restarts with a strictly larger deviation magnitude.

Let $\rho_k \triangleq \mathbb{E}_r[\|\delta(\theta_k; r)\|^2] \big/ \mathbb{E}_r[\|g(\theta_k; r)\|^2]$ denote the fraction of gradient variance attributable to the deviation component. The feedback loop drives $\rho_k$ monotonically upward: early in training, $\rho_k$ remains small since $\|g^*\|$ dominates; as $\|g^*\| \to 0$ and $\|\delta\|$ grows via the loop, $\rho_k \to 1$. At this point, parameter updates are overwhelmingly determined by the leakage signal, and training collapses.

\subsubsection{The Trilemma}

The analyses of Strategies~A and~B, together with the gradient decomposition, yield the central impossibility result of on-policy self-distillation.

\begin{theorem}[Impossibility Trilemma]
\label{thm:trilemma}
In any distribution-matching framework where teacher and student share parameters (i.e., $P_T$ and $P_S$ are parameterized by the same $\theta$), the following three properties cannot hold simultaneously:
\begin{enumerate}[label=(\alph*),nosep]
    \item \textbf{Objective stability}: the optimization target does not drift between successive steps ($\Delta_T = 0$).
    \item \textbf{Sustained improvement}: the distillation signal does not vanish ($\|\nabla_\theta \mathcal{L}\| \not\to 0$).
    \item \textbf{Leakage-free training}: the deviation component does not drive parameter drift.
\end{enumerate}
\end{theorem}

Strategy~A (frozen teacher) satisfies (a) by construction ($P_T^{\theta_0}$ is fixed, hence $\Delta_T = 0$), but violates (b) by Proposition~\ref{prop:frozen-ceiling}. Strategy~B (online teacher) satisfies (b) since the teacher continually evolves to provide non-vanishing signal, but violates (a) by Theorem~\ref{thm:instability}. Property~(c) is violated under \textit{either} strategy whenever $I(Y_t; R \mid X) > 0$: by Proposition~\ref{prop:grad-decomp}(ii), the per-sample deviation variance $\mathbb{E}_r[\|\delta\|^2]$ is strictly positive, independently of the teacher management scheme. Under Strategy~B, Proposition~\ref{prop:feedback-loop} further ensures that $\rho_k$ is monotonically non-decreasing, so the leakage intensifies over the course of training. Hybrid strategies (e.g., periodic teacher snapshots) merely interpolate between (a) and (b): within each update interval, (a) holds but (b) progressively degrades as the teacher becomes stale; refreshing the teacher restores (b) while disrupting (a). Property~(c) remains unsatisfied throughout, as the underlying mutual information gap persists regardless of the snapshot schedule.


\subsection{Bayesian Interpretation of RLSD Evidence Weights}
\label{app:bayesian}

This section provides a formal probabilistic derivation of the RLSD evidence ratio $w_t = P_T(y_t)/P_S(y_t)$ introduced in \S\ref{sec:rlsd}, establishing that it corresponds to a sequential Bayesian belief update.

\subsubsection{Modeling Assumption}

\begin{assumption}[Consistent Conditional Approximation]
\label{asm:consistent}
The shared model $\pi_\theta$ provides a consistent approximation to the true conditional distributions in the following sense:
\begin{align}
    P_S(y_t \mid x, y_{<t}) &\triangleq \pi_\theta(y_t \mid x, y_{<t}) \approx P(y_t \mid x, y_{<t}), \\
    P_T(y_t \mid x, r, y_{<t}) &\triangleq \pi_\theta(y_t \mid x, r, y_{<t}) \approx P(y_t \mid x, r, y_{<t}).
\end{align}
\end{assumption}

When the model has sufficient capacity and the privileged information $r$ is provided via in-context conditioning, Assumption~\ref{asm:consistent} is reasonable: a single architecture can faithfully represent both the prior predictive distribution (without $r$) and the posterior predictive distribution (with $r$). This parallels the standard assumption in Bayesian deep learning that a sufficiently expressive model family can approximate the true posterior predictive distribution.

\subsubsection{The Evidence Ratio as Belief Update}

\begin{theorem}[RLSD Weights as Belief Update Ratios]
\label{thm:bayesian}
Under Assumption~\ref{asm:consistent}, the token-level evidence ratio satisfies the identity:
\begin{equation}
    w_t = \frac{P_T(y_t \mid x, r, y_{<t})}{P_S(y_t \mid x, y_{<t})} = \frac{P(r \mid x, y_{\leq t})}{P(r \mid x, y_{<t})},
    \label{eq:belief-update}
\end{equation}
i.e., $w_t$ equals the ratio of the posterior belief in the privileged information $r$ after observing token $y_t$ to the belief before observing $y_t$.
\end{theorem}

\begin{proof}
Apply Bayes' theorem to the joint distribution $P(r, y_t \mid x, y_{<t})$:
\begin{equation}
    P(y_t \mid x, r, y_{<t}) = \frac{P(r, y_t \mid x, y_{<t})}{P(r \mid x, y_{<t})} = \frac{P(r \mid x, y_{\leq t}) \cdot P(y_t \mid x, y_{<t})}{P(r \mid x, y_{<t})}.
\end{equation}
Dividing both sides by $P(y_t \mid x, y_{<t})$ and applying Assumption~\ref{asm:consistent}:
\begin{equation}
    \frac{P_T(y_t \mid x, r, y_{<t})}{P_S(y_t \mid x, y_{<t})} \approx \frac{P(y_t \mid x, r, y_{<t})}{P(y_t \mid x, y_{<t})} = \frac{P(r \mid x, y_{\leq t})}{P(r \mid x, y_{<t})}. \qedhere
\end{equation}
\end{proof}

\subsubsection{Interpretation}

Theorem~\ref{thm:bayesian} reveals that the RLSD weight $w_t$ is a \textit{sequential Bayesian evidence ratio}: it quantifies the degree to which generating token $y_t$ updates the belief that the privileged information $r$ is consistent with the current trajectory.

Concretely, $w_t > 1$ indicates that generating $y_t$ \textit{strengthens} the posterior belief in $r$, constituting positive evidence for the correct reasoning path. For instance, when solving $2x + 3 = 7$, writing ``$2x = 4$'' substantially increases the belief that the trajectory leads to the correct answer $x = 2$; RLSD assigns such tokens elevated positive credit. $w_t < 1$ indicates that $y_t$ \textit{weakens} the belief in $r$, constituting negative evidence. An erroneous step such as ``$5x = 7$'' would sharply reduce the belief; in an incorrect trajectory, such tokens bear greater blame. $w_t = 1$ indicates that $y_t$ is informationally neutral with respect to $r$, as is typical of formatting connectives (e.g., ``therefore,'' ``we have'') that carry no reasoning content; RLSD assigns these tokens the baseline advantage.

\subsubsection{Telescoping to Sequence-Level Evidence}

Taking the product over all token positions yields a telescoping identity:
\begin{equation}
    \prod_{t=1}^{T} w_t = \prod_{t=1}^{T} \frac{P(r \mid x, y_{\leq t})}{P(r \mid x, y_{<t})} = \frac{P(r \mid x, y)}{P(r \mid x)}.
\end{equation}
The right-hand side is the sequence-level Bayesian evidence ratio: the total belief update after observing the complete trajectory $y$. The per-token weights $\{w_t\}_{t=1}^T$ thus provide a fine-grained \textit{decomposition} of this sequence-level evidence into individual token contributions, elevating the granularity of credit assignment from the sequence level (as in GRPO) to the token level.

\subsubsection{Contrast with OPSD: Logical Attribution vs.\ Behavioral Cloning}

Theorem~\ref{thm:bayesian} clarifies a fundamental distinction between RLSD and distribution matching methods.

OPSD performs \textit{behavioral cloning}: it requires the student to replicate the teacher's output distribution $P_T(\cdot \mid x, r)$ at every token position. If the teacher, having observed the reference solution, favors a particular phrasing (e.g., ``according to the hint''), OPSD drives the student to adopt the same phrasing, even when it carries no reasoning value. The objective is to minimize the \textit{distributional distance} between $P_S$ and $P_T$.

RLSD performs \textit{logical credit attribution}: it does not require the student to imitate any specific token choice of the teacher. The weight $w_t$ measures a single property: whether the token generated by the student constitutes positive Bayesian evidence for the correct answer $r$. As long as a student-generated token increases the posterior belief in $r$ ($w_t > 1$), it receives elevated credit regardless of whether it matches the teacher's preferred wording. RLSD is concerned not with \textit{what} the student says, but with the \textit{informational contribution} of what it says toward the correct derivation.

In information-theoretic terms, $w_t$ quantifies the \textit{per-token information gain} toward the correct answer. RLSD allocates credit accordingly: tokens that advance the reasoning are rewarded, tokens that derail it are penalized, and tokens that are informationally inert receive the baseline advantage. No imitation of the teacher's distributional shape is involved at any point.

\subsubsection{Connection to the OPSD Gradient}

It is instructive to note that the importance weight $P_T(v \mid r)/P_S(v)$ appearing in the OPSD per-sample gradient (Proposition~\ref{prop:grad-decomp}) is mathematically identical to the evidence ratio $w_t$ used in RLSD. This is not coincidental: both methods operate on the same underlying quantity, but employ it in fundamentally different ways.

In OPSD, $P_T/P_S$ serves as a \textit{gradient weight} that spans the entire vocabulary $\mathcal{V}$, driving $P_S$ to match the shape of $P_T$. The per-sample deviation $[P_T(v \mid r) - \bar{P}_T(v)]$ is inseparable from the expected signal (Proposition~\ref{prop:grad-decomp}), causing $r$-specific patterns to infiltrate parameter updates.

In RLSD, $P_T/P_S$ serves as a \textit{stop-gradient scalar credit multiplier} applied only to the on-policy sampled token $y_t$. It modulates the \textit{magnitude} of the existing advantage but does not enter the gradient direction. The privileged information influences only how much credit each token receives, not which tokens are reinforced or penalized, nor the direction of parameter updates.

It is precisely this change in usage, from distribution matching to credit assignment, from gradient weight to scalar multiplier, that transforms the same mathematical quantity from the source of leakage into a tool for precise credit attribution.

\subsubsection{Limitations of the Bayesian Approximation}
\label{app:bayesian-limitations}

The Bayesian interpretation (Theorem~\ref{thm:bayesian}) relies on Assumption~\ref{asm:consistent}, which may not hold exactly in practice. When the model's in-context conditioning is imperfect, e.g., when $r$ is excessively long or the model has limited capacity, the evidence ratio $P_T/P_S$ becomes a noisy approximation of the true belief update ratio.

However, the structural guarantees of RLSD \textit{do not depend on Assumption~\ref{asm:consistent}}. The direction anchoring property ($\mathrm{sign}(\hat{A}_t) \equiv \mathrm{sign}(A)$, \S\ref{sec:rlsd}) and the support set isolation (gradients act only on student-sampled tokens) are consequences of the algorithm design (stop-gradient, on-policy sampling, scalar modulation, clipping) and hold regardless of the quality of the Bayesian approximation. Similarly, the leakage-free guarantee established in Appendix~\ref{app:leakage-free} and the resolution of the impossibility trilemma (Appendix~\ref{app:trilemma}) are purely structural results. This separation ensures that all theoretical guarantees remain valid even when the model is an imperfect Bayesian agent.


\subsection{Formal Leakage-Free Guarantee for RLSD}
\label{app:leakage-free}

Building on the Bayesian interpretation in Appendix~\ref{app:bayesian} and the diagnosis of distribution matching in \S\ref{sec:ill-posed}, this section establishes that RLSD is structurally immune to privileged information leakage and verifies that it circumvents all three deficiencies identified in \S\ref{sec:gradient-structure} and Appendix~\ref{app:trilemma}.

\subsubsection{Gradient Structure of RLSD}

The RLSD policy gradient takes the form:
\begin{equation}
    \nabla_\theta \mathcal{J}_{\text{RLSD}}(\theta) = \mathbb{E}_{x}\, \mathbb{E}_{y \sim \pi_\theta(\cdot \mid x)} \left[ \frac{1}{G}\sum_{i=1}^G \frac{1}{|y^{(i)}|}\sum_{t=1}^{|y^{(i)}|} \hat{A}_t^{(i)} \cdot \nabla_\theta \log \pi_\theta(y_t^{(i)} \mid x, y_{<t}^{(i)}) \right],
\end{equation}
where the token-level advantage is constructed as specified in Eqs.~\eqref{eq:delta}\textendash\eqref{eq:final}. This gradient expression differs from the OPSD gradient in two structural properties.

\noindent \textbf{Property 1: Direction anchoring.}
The token-level advantage $\hat{A}_t$ strictly preserves the sign of the sequence-level advantage $A$:
\begin{equation}
    \mathrm{sign}(\hat{A}_t) = \mathrm{sign}(A) \quad \forall\, t,
\end{equation}
since $\mathrm{clip}(w_t, 1 - \epsilon_w, 1 + \epsilon_w) > 0$. This ensures that the environment reward has \textit{exclusive authority} over the direction of policy updates: all tokens in correct trajectories are reinforced, and all tokens in incorrect trajectories are penalized. The teacher's privileged information, channeled through $w_t$, modulates only the \textit{relative intensity} of reinforcement or penalization across tokens within a trajectory, but can never flip the direction. In contrast, the OPSD gradient direction at each token position is determined by $P_T(v \mid r)$ or its difference from $P_S(v)$, so the teacher's privileged preference directly controls \textit{where} the gradient points.

\noindent \textbf{Property 2: Student-determined support.}
The expectation in the RLSD gradient is taken over trajectories $y \sim \pi_\theta(\cdot \mid x)$, sampled from the student's own policy without access to $r$. The log-derivative $\nabla_\theta \log \pi_\theta(y_t \mid x, y_{<t})$ is evaluated only at tokens the student itself generates. The gradient's support, defined as the set of token sequences receiving non-zero gradient signal, is therefore identical to the support of $\pi_\theta(\cdot \mid x)$.

In contrast, the OPSD per-sample gradient $g(\theta; r) = -\sum_{v \in \mathcal{V}} P_T(v \mid r) \nabla_\theta \log P_S(v)$ ranges over the entire vocabulary $\mathcal{V}$: tokens that the teacher strongly favors due to $r$ receive $P_T(v \mid r)$-weighted gradient contributions, actively pulling unseen privileged patterns into the parameter updates. RLSD eliminates this channel entirely.

\subsubsection{Zero-Leakage Guarantee}

Combining the two structural properties yields the following guarantee.

\begin{theorem}[Leakage-Free Training Under RLSD]
\label{thm:leakage-free}
In the RLSD policy gradient objective, the privileged information $r$ enters the computation only through the stop-gradient scalar $w_t$, and $w_t$ is isolated at three levels:
\begin{enumerate}[label=(\roman*),nosep]
    \item \textbf{Directional isolation.} By Property~1, $\mathrm{sign}(\hat{A}_t) = \mathrm{sign}(A)$, so $r$ cannot influence the gradient sign of any token. The direction of parameter updates in $\theta$-space is fully determined by the student-sampled token $y_t$ (via $\nabla_\theta \log \pi_\theta(y_t \mid x, y_{<t})$) and the environment correctness judgment (via $\mathrm{sign}(A)$), neither of which contains information about $r$.
    \item \textbf{Support isolation.} By Property~2, any token $y_{\mathrm{leak}} \notin \mathrm{supp}(\pi_\theta(\cdot \mid x, y_{<t}))$ that exists only in the privileged mode has strictly zero sampling probability and contributes zero expected gradient.
    \item \textbf{Magnitude boundedness.} The weight $w_t$ is clipped to $[1 - \epsilon_w, 1 + \epsilon_w]$ and further attenuated by the mixing coefficient $\lambda$, strictly bounding its influence. As the student internalizes reasoning ability and $P_S \to P_T$, the weights satisfy $w_t \to 1$, and RLSD automatically degrades to vanilla GRPO.
\end{enumerate}
Consequently, there exists no mathematical pathway through which patterns from the privileged information $r$ can be injected into the direction of parameter updates. The privileged information in RLSD influences only the \textit{magnitude} of token-level credit, not \textit{which} tokens are reinforced or penalized, nor \textit{whether} reinforcement or penalization occurs.
\end{theorem}

\subsubsection{Resolution of the Impossibility Trilemma}

Appendix~\ref{app:trilemma} establishes that distribution matching under shared parameters cannot simultaneously satisfy objective stability (a), sustained improvement (b), and leakage-free training (c). We verify that RLSD satisfies all three.

\noindent \textbf{(a) Objective stability.}
RLSD optimizes the environment reward $R(x, y) \in \{0, 1\}$, an external signal independent of $\theta$. The teacher contributes only through the stop-gradient scalar $w_t$, which does not constitute part of the optimization objective. The teacher drift term $\Delta_T$ from Theorem~\ref{thm:instability} is therefore inapplicable: the optimization anchor is a fixed external reward, not a drifting internal distribution. Even as $\theta$ updates cause $w_t$ to change at subsequent steps, this affects only the credit magnitude, not the gradient direction (Property~1) or support (Property~2).

\noindent \textbf{(b) Sustained improvement.}
RLSD is architecturally designed as a \textit{curriculum-aware} training strategy that seamlessly unifies two complementary learning phases. In the early stage, when the student's policy is far from optimal and token-level credit discrimination is most valuable, the self-distillation signal provides dense, fine-grained guidance that accelerates convergence well beyond what uniform advantages can achieve. As training progresses, the mixing coefficient $\lambda$ is linearly decayed to zero, smoothly transitioning RLSD into vanilla GRPO. This scheduled transition reflects a principled design choice: the teacher's dense credit assignment is most beneficial when the student's policy is still coarse-grained and requires fine-grained steering, whereas in the later stage, the environment reward alone suffices to drive continued improvement. Crucially, the GRPO policy gradient remains non-zero as long as the student has not reached the optimal policy, ensuring that training does not stall after the transition. This two-phase structure, dense credit assignment for rapid early improvement followed by standard RLVR for sustained long-term optimization, allows RLSD to inherit the convergence ceiling of GRPO while substantially accelerating the path toward it. The explicit scheduling of $\lambda$ also provides practitioners with a transparent and controllable mechanism to balance the intensity of teacher guidance against training stability, making RLSD a practical and fully compatible enhancement to existing RLVR pipelines.

\noindent \textbf{(c) Leakage-free training.}
By Theorem~\ref{thm:leakage-free}, the gradient direction is determined entirely by the environment reward and the student's on-policy samples. The deviation component $\delta$ from Proposition~\ref{prop:grad-decomp} does not arise, because RLSD does not perform distribution matching and therefore does not generate the $[P_T(v \mid r) - \bar{P}_T(v)]$ bias that drives the gradient direction in OPSD. The self-reinforcing feedback loop of Proposition~\ref{prop:feedback-loop} is likewise absent: even if shared parameters cause $S(\theta)$ to change, this affects only the numerical value of $w_t$, which, after clipping, enters $\hat{A}_t$ as a bounded scalar multiplier that does not alter the gradient direction or the token-level support.

Table~\ref{tab:trilemma} summarizes the trilemma from the perspective of both OPSD strategies and RLSD.

\begin{table}[h]
\centering
\caption{Comparison of OPSD parameter management strategies and RLSD across the three desiderata identified in the impossibility trilemma.}
\label{tab:trilemma}
\small
\begin{tabular}{lccc}
\toprule
\textbf{Property} & \textbf{OPSD (Frozen Teacher)} & \textbf{OPSD (Online Teacher)} & \textbf{RLSD} \\
\midrule
(a) Objective stability & $\checkmark$ & $\times$ & $\checkmark$ \\
(b) Sustained improvement & $\times$ & $\checkmark$ & $\checkmark$ \\
(c) Leakage-free training & $\times$ & $\times$ & $\checkmark$ \\
\bottomrule
\end{tabular}
\end{table}

\newpage

\end{document}